\documentclass[12pt,a4paper]{article}

\usepackage[nonatbib]{neurips_2018}
\usepackage[utf8]{inputenc} 
\usepackage[T1]{fontenc}    
\usepackage[english]{babel}
\usepackage{url}            
\usepackage{nicefrac}       
\usepackage{microtype}      
\usepackage[dvips]{graphicx}
\usepackage{amsmath,amsfonts,amsthm,amssymb}
\usepackage{mathrsfs}
\usepackage[usenames, dvipsnames]{color}
\usepackage{xspace}
\usepackage{adjustbox}
\usepackage[colorlinks=true,citecolor=blue,linkcolor=blue,urlcolor=blue,pdfstartview=FitH]{hyperref}
\usepackage{nameref}
\usepackage{enumitem}
\usepackage{tabularx, booktabs}
\usepackage{footnote}
\makesavenoteenv{tabular}
\usepackage{algorithm}
\usepackage{algorithmicx}
\usepackage{algpseudocode}
\usepackage{scrextend}
\usepackage{array}
\usepackage{environ}

\makeatletter
\renewcommand{\ALG@name}{Implementation}
\makeatother

\usepackage{caption}

\def\om_prop#1{\textcolor{Purple}{#1}}

\usepackage{tikz,tkz-tab}
\usetikzlibrary{fit,calc,positioning,trees}
\usetikzlibrary{arrows}
\tikzstyle{level 1}=[level distance=4cm, sibling distance=2.5cm]
\tikzstyle{level 2}=[level distance=8cm, sibling distance=0.6cm]
\tikzstyle{bag} = [text width=4em, text centered]
\tikzstyle{end} = [circle, minimum width=3pt,fill, inner sep=0pt]
\usetikzlibrary{decorations.pathmorphing}
\usetikzlibrary{decorations.pathreplacing}
\usetikzlibrary{decorations.shapes}
\usetikzlibrary{decorations.text}
\usetikzlibrary{decorations.markings}
\usetikzlibrary{decorations.fractals}
\usetikzlibrary{decorations.footprints}

\makeatletter
\newcommand{\manuallabel}[2]{\def\@currentlabel{#2}\label{#1}}
\makeatother

\newcommand\numberthis{\addtocounter{equation}{1}\tag{\theequation}}

\algnewcommand{\Inputs}[1]{%
  \State \textbf{Inputs:}
  \Statex \hspace*{\algorithmicindent}\parbox[t]{.8\linewidth}{\raggedright #1}
}

\algnewcommand{\Initialize}[1]{%
  \State \textbf{Initialize:}
  \Statex \hspace*{\algorithmicindent}\parbox[t]{.8\linewidth}{\raggedright #1}
}

\NewEnviron{myequation}{%
    \begin{equation}
    \scalebox{1.5}{$\BODY$}
    \end{equation}
    }

\makeatletter
\newcommand*{\inlineequation}[2][]{%
  \begingroup
    \refstepcounter{equation}%
    \ifx\\#1\\%
    \else
      \label{#1}%
    \fi
    \relpenalty=10000 %
    \binoppenalty=10000 %
    \ensuremath{%
      #2 %
    }%
    ~\@eqnnum
  \endgroup
}
\makeatother

\makeatletter
\newcommand\footnoteref[1]{\protected@xdef\@thefnmark{\ref{#1}}\@footnotemark}
\makeatother

\graphicspath{{image_cj/}}
\DeclareGraphicsExtensions{.pdf,.png, jpg}

\DeclareMathOperator*{\argmin}{arg\,min}
\DeclareMathAlphabet{\mathpzc}{OT1}{pzc}{m}{it}

\def\Esp{\mathbb{E}}
\def\Var{\mathrm{Var}}

\def\é{\'{e}}
\def\è{\`{e}}
\def\ê{\^{e}}
\def\à{\`{a}}
\def\ô{\^{o}}

\newcolumntype{C}[1]{>{\centering\arraybackslash}p{#1}}
\newcolumntype{L}[1]{>{\raggedleft\arraybackslash}p{#1}}
\newcolumntype{R}[1]{>{\raggedright\arraybackslash}p{#1}}

\newtheorem{theo}{Theorem}
\newtheorem{prop}{Proposition}
\newtheorem{lem}{Lemma}
\newtheorem{Assumption}{Assumption}

\newtheorem{alg}{Algorithm}
\newtheorem{rem}{Remark}

\title{Improving reinforcement learning algorithms:\\
 towards optimal learning rate policies}
\author{Othmane Mounjid\thanks{\'{E}cole Polytechnique, CMAP} , Charles-Albert Lehalle \thanks{Capital Fund Management, Paris and Imperial College, London}{}}
\date{\today\\}
\begin{document}
\maketitle

\begin{abstract}
This paper investigates to what extent one can improve reinforcement learning (RL) algorithms. Our study is split in three parts. First, our analysis shows that the classical asymptotic convergence rate
$O(1/\sqrt{N})$ is pessimistic and can be replaced by $O((\log(N)/N)^{\beta})$ with
$\frac{1}{2}\leq \beta \leq 1$, and $N$ the number of iterations. Second, we propose a dynamic
optimal policy for the choice of the learning rate $(\gamma_k)_{k\geq 0}$ used in RL. We decompose our policy into two interacting levels: the inner and the outer
level. In the inner level, we present the \nameref{Alg:v_4_s} algorithm (for ``PAst
  Sign Search'') which, based on a predefined sequence $(\gamma^o_k)_{k\geq 0}$, constructs a new sequence $(\gamma^i_k)_{k\geq 0}$ whose error decreases faster. In the outer level, we propose an optimal methodology for the selection of the predefined sequence $(\gamma^o_k)_{k\geq 0}$. Third, we show empirically that our selection methodology of the learning rate outperforms significantly standard algorithms used in RL for the three following applications: the estimation of a drift, the optimal placement of limit orders, and the optimal execution of large number of shares.
\end{abstract}

\section{Introduction}
\label{sec:Intro}

We consider a discrete state space $\mathcal{Z} = \mathbb{N}$ or $\mathcal{Z} = \{1,\ldots,d\}$ with $d \in \mathbb{N}^*$. We are interested in finding $q^* \in \mathcal{Q} \subset \mathbb{R}^{\mathcal{Z}}$ solution of
\begin{equation}
M(q,z) = \Esp[m(q,X(z),z)] = 0, \qquad \forall z \in \mathcal{Z},
\label{Eq:init_pbm}
\end{equation}
where $X(z) \in \mathcal{X}$ is a random variable with an unknown distribution, and $m$ is a function from $\mathcal{Q} \times \mathcal{X} \times \mathcal{Z}$ to $\mathbb{R}$. Although the distribution of $X(z)$ is unspecified, we assume that we can observe some variables $(Z_n)_{n\geq 0}$ valued in $\mathcal{Z}$ and $\big(X_{n+1}(Z_{n})\big)_{n\geq 1}$ drawn from the same distribution of $X(Z_{n})$. Reinforcement learning (RL) addresses this problem through the following iterative procedure:
\begin{equation}
q_{n+1}(Z_n) = q_n (Z_n) - \gamma_n(Z_n) m(q_n,X_{n+1}(Z_n),Z_n), \quad \forall n \geq 0,
\label{Eq:RobinM_RL}
\end{equation}
where $q_0$ is a given initial condition, and each $\gamma_n$ is a component-wise non-negative vector valued in $\mathbb{R}^{\mathcal{Z}}$. The connection between RL, problem \eqref{Eq:init_pbm}, and Algorithm \eqref{Eq:RobinM_RL} is detailed in Section \ref{sec:rel_RL_SA}. It is possible to recover the classical SARSA, $Q$-learning, and double $Q$-learning algorithms used in RL by taking a specific expression for $m$ and $X_{n+1}$. Note that Algorithm \eqref{Eq:RobinM_RL} is different from the standard Robbins-Monro (RM) algorithm used in stochastic approximation (SA)
\begin{equation}
q_{n+1} =  q_n - \gamma_n \bar{m}(q_n,X_{n+1}),
\label{Eq:RobinM}
\end{equation}
with $\bar{m} \in \mathbb{R}^{\mathcal{Z}}$ whose $z$-th coordinate is defined such that $\bar{m}(q,x)(z) = m(q,x(z),z)$ for any $z \in \mathcal{Z}$ and $\gamma_n \geq 0$, mainly because, as it is frequent in RL, we do not observe the entire variable $\big(X_{n+1}(z)\big)_{z\in \mathcal{Z}})$ but only its value according to the coordinate $Z_n$. Indeed, the way $(Z_n)_{n\geq 1}$ visits the set $\mathcal{Z}$ plays a key role in the convergence of Algorithm \eqref{Eq:RobinM_RL}. RM algorithm was first introduced by Robbins and Monro in \cite{robbins1951stochastic}. After that, it was studied by many authors who prove the convergence of $q_n$ towards $q^*$, see \cite{benveniste1987and,bertsekas1996neuro,blum1954approximation,kushner1978stochastic}. The asymptotic convergence rate has also been investigated in many papers, see \cite{benveniste1987and,kushner2003stochastic,sacks1958asymptotic}. They show that this speed is in general proportional to $1/{\sqrt{N}}$ with $N$ the number of iterations.\\

In this work, we give a special focus to RL problems. Nowadays RL cover a very wide collection of recipes to solve control problems in an exploration-exploitation context. This literature started in the seventies, see \cite{watkins1989learning,werbos1987building}, and became famous mainly with the seminal paper of Sutton, see \cite{sutton1998introduction}. It largely relied on the recent advances in the control theory developed in the late 1950s, see \cite{bellman1959functional}. The key tool borrowed from this theory is the dynamic programming principle satisfied by the value function. This principle enables us to solve control problems numerically when the environment is known and the dimension is not too large. To tackle the curse of dimensionality, recent papers, see \cite{sirignano2018dgm}, use deep neural networks (DNN). For example, in \cite{hutchinson1994nonparametric}, authors use DNN to derive optimal hedging strategies for finance derivatives and in \cite{manziukAlgorithmicTradingReinforcement2019} they use a similar method to solve a high dimensional optimal trading problem. To overcome the fact that environment is unspecified, it is common to use an RM type algorithm which estimates on-line quantities of interest. The combination of control theory and SA techniques gave birth to numerous papers on RL.\\

Our contributions are as follows. 
\begin{itemize}
    \item First, we conduct an error analysis to show that the classical asymptotic rate $O(1/\sqrt{N})$ is pessimistic and can be enhanced in many situations. For this, we borrow tools from the statistical learning theory and show how to use them in a RL setting to get a $O((\log(N)/N)^{\beta})$ asymptotic speed with $1/2 \leq \beta \leq 1$ and $N$ the number of iterations.
    \item Second, we propose a dynamic policy for the choice of the step size $(\gamma_k)_{k\geq 0}$ used in \eqref{Eq:RobinM_RL}. Our policy is decomposed into two interacting levels: the inner and the outer level. In the inner level, we introduce the \nameref{Alg:v_4_s} algorithm, for ``PAst Sign Search''. This algorithm builds a new sequence $(\gamma^i_k)_{k\geq 0}$, using a predefined sequence $(\gamma^o_k)_{k\geq 0}$ and the sign variations of $m(q_n,X_{n+1}(Z_n),Z_n)$. The error of $(\gamma^i_k)_{k\geq 0}$ decreases faster than the one of $(\gamma^o_k)_{k\geq 0}$. In the outer level, we present an optimal dynamic policy for the choice of the predefined sequence $(\gamma^o_k)_{k\geq 0}$. These two levels are interacting in the sense that \nameref{Alg:v_4_s} influences the construction of $(\gamma^o_k)_{k\geq 0}$.
    \item Third, convergence of \nameref{Alg:v_4_s} algorithm is established and error bounds are provided.
    \item Finally, we show that our selection methodology provides better convergence results than standard RL algorithms in three numerical examples: the drift estimation, the optimal placement of limit orders, and the optimal execution of a large number of shares. When needed the proofs of convergence of our numerical methods are given.
\end{itemize}  

The structure of this paper goes as follows: Section \ref{sec:rel_RL_SA} describes the relation between RL and Equation \eqref{Eq:init_pbm}. Section \ref{sec:conv:rate} reformulates \eqref{Eq:init_pbm} as an optimization problem and defines with accuracy the different sources of error. This enables us to derive the convergence speed of RL algorithms. Section \ref{sec:optim:rate} contains our adaptive learning rate algorithm. Finally, Section \ref{sec:examples} provides numerical examples taken from the optimal trading literature: optimal placement of a limit order, and the optimization of the trading speed of a liquidation algorithm. Proofs and additional results are relegated to an appendix.

\section{Reinforcement learning}
\label{sec:rel_RL_SA}

We detail in this section the relation between \eqref{Eq:init_pbm} and RL since we are interested in solving RL problems. RL aims at estimating the $Q$-function which quantifies the value for the player to choose the action $a$ when the system is at $s$. Let $t$ be the current time, $U_t \in \mathcal{U}$ be a process defined on a filtered probability space $(\Omega,\mathcal{F},\mathcal{F}_t,\mathbb{P})$ which represents the current state of the system, and $A_t \in \mathcal{A}$ the agent action at time $t$. We assume that the process $(U_t,A_t)$ is Markov. The agent aims at maximizing
\begin{align*}
\Esp[\int_{0}^T \rho^s f(s,U_s,A_s) \, ds + \rho^{T}g(U_T)], \numberthis \label{Eq:OjectFunc}
\end{align*}
with $g$ the terminal constraint, $f$ the instantaneous reward, $\rho$ a discount factor, and $T$ the final time. Let us fix a time step $\Delta > 0$ and allow the agent to take actions only at times\footnote{We recall the following classical result: when $\Delta$ goes to zero, the value function and the optimal control of this problem converges towards the one where decisions are taken at any time.} $k\Delta$ with $k\in \mathbb{N}$. The $Q$-function is defined as follows:
\begin{align*}
Q(t,u,a) = \sup_{A } \Esp_{A}[\int_{t}^T \rho^{(s-t)}f(s,U_s,A_s) \, ds + \rho^{(T-t)}g(U_T) |U_t =u,A_t = a], 
\end{align*}
with $(t,u,a)\in \mathbb{R}_{+} \times \mathcal{U} \times \mathcal{A}$, $A = \{A_t \, , t < T\}$ a possible control process for the agent. Note that the action of the agent depend on $s = (t,u)$ with $t$ the current time and $u$ the current state. We view the agent control $A$ as a feedback process (i.e adapted to the filtration $\mathcal{F}_t$). The $Q$-function satisfies the classical dynamic programming principle (DPP)  
\begin{align*}
Q(t,u,a) =   \Esp [ R_{t+\Delta} + \rho^{\Delta} \underset{a' \in \mathcal{A}}{\sup} \, Q(t+\Delta,U_{t+\Delta},a')|U_t = u, A_t = a], \numberthis \label{Eq:Bellman}
\end{align*}
with $ R_{t+\Delta} = \int_{t}^{t + \Delta} \rho^{(s-t)}f(s,U_s,A_s)\, ds$. Equation \eqref{Eq:Bellman} shows that the optimal expected gain when the agent starts at $s$ and chooses action $a$ at time $t$ is the sum of the next expected reward $R_{t+\Delta}$ plus the value of acting optimally starting from the new position $U_{t+\Delta}$ at time $t+\Delta$. By reformulating \eqref{Eq:Bellman}, we obtain that $Q$ solves the following equation:
\begin{equation}
\Esp[m(q,X(z),z)] = 0, \qquad \forall z = (t,u,a) \in \mathcal{Z} = [0,T] \times \mathcal{U} \times \mathcal{A}, \label{Eq:Rel_RL_SA}
\end{equation}
where 
\begin{itemize}
    \item $X(z) = (U^{z}_{t+\Delta},R^{z}_{t+\Delta}) \in \mathcal{X} = \mathcal{U} \times \mathbb{R}$, $U^{z}_{s}$ and $R^{z}_{s}$ are respectively the conditional random variables $U_s$ and $R_s$ given the initial condition $(U_t,A_t) = (u,a)$ with $z = (t,u,a) \in \mathcal{Z}$.
    \item $m$ is defined as follows:
$$
m(q,x,z^1) = H(q,x,z^1) - q(z^1), \qquad H(q,x,z^1) = r + \rho^{\Delta} \underset{a' \in \mathcal{A} }{\sup} \, q(t^1+\Delta,u,a'),
$$
for any $x = (u,r) \in \mathcal{X}$, and $z^1 = (t^1,u^1,a^1)  \in \mathcal{Z}$.
\end{itemize} 

 Thus, one can use \eqref{Eq:init_pbm} to solve \eqref{Eq:Rel_RL_SA}.

Note that Equation \eqref{Eq:Rel_RL_SA} shows that one can study $Q$ only on the time grid\footnote{Here, we take $T = n^* \Delta$ with of $n^* \in \mathbb{N}^*$. Such approximation is not restrictive.} $D_T = \{n\Delta, \, n \leq T/\Delta\}$. Thus, we define $A_k$ and $U_k$ such that $A_k = A_{k\Delta}$ and $U_k = U_{k \Delta}$ for any $k \in \mathbb{N}$. The key variable to study is not the agent decision $A_{k}$ but $Z_k =(k,U_k,A_k)$. Thus, the rest of the paper formulates the results in terms of $Z_k$ only.

\paragraph{Actions of the agent.} It is important in practice to visit the space $D_T \times \mathcal{U} \times \mathcal{A}$ sufficiently enough. Thus, to learn $Q$, it is common to not choose the maximising action\footnote{The maximising action $a^*$ for a state $u$ is defined such that $a^* = \arg\max_{a \in \mathcal{A}} q(u,a)$.}, but to encourage exploration by visiting the states where the error is large. We give in Appendix \ref{sec:actions_agent} examples of policies that promote exploration and others that maximize the $Q$-function. 

\begin{rem} In general, a solution $q^*$ of \eqref{Eq:init_pbm} does not necessarily solve an optimization problem in the form of \eqref{Eq:OjectFunc}. However, when $M$ can be written as the gradient of some given function $f$ (i.e. $\nabla f =M$), $q^*$ becomes a solution of a problem in the form of \eqref{Eq:OjectFunc}.
\label{rem:rem1}
\end{rem}

\section{Improvement of the asymptotic convergence rate}
\label{sec:conv:rate}

In \cite[Part 2, Section 4]{benveniste2012adaptive}, \cite[Section 10]{kushner2003stochastic} and \cite[Section 7]{kushner1978stochastic}, the authors show a central limit theorem for the procedure \eqref{Eq:RobinM} which ensures a convergence rate of $O(1/\sqrt{N})$ where $N$ is the number of iterations. In this section, we extend such convergence rate to Algorithm \eqref{Eq:RobinM_RL} and aim at understanding how one can improve it. For this, we decompose our total error into two standard components: estimation error and optimization error. 
\subsection{Error decomposition}
\label{sec:2_1_error decomposition}
In this section, the space $\mathcal{Z} = \{1,\ldots,d\}$ is finite with $d \in \mathbb{N}^*$. In such case, we view $q$, and $M(q)$ as vectors of $\mathbb{R}^\mathcal{Z}$. Moreover, the process $(Z_n)_{n\geq  1}$ is an homogeneous Markov chain. We consider the following assumption.
\begin{Assumption}[Existence of a solution]
There exists a solution $q^*$ of Equation \eqref{Eq:init_pbm}.
\label{assump:assump__0}
\end{Assumption}
Under Assumption \ref{assump:assump__0}, the function $q^*$ is a solution to the minimization problem
\begin{align*}
\min_{q \in \mathcal{Q}} g(q) , \numberthis \label{Eq:Min_pbm}
\end{align*}
where $g$ can be selected as follows:
\begin{itemize}
    \item If $M$ can be written as the gradient of some function $f$, see Remark \ref{rem:rem1}, one can take $g = f$.
    \item Otherwise, it is always possible to set $g(q) = \|M(q)\|$. For simplicity, we place ourselves in this case for the rest of the section.
\end{itemize}

In our context we do not have a direct access to the distribution of $X(z)$. Nevertheless, we assume that at time $n$ we keep a memory of a training sample of $n(z)$ independent variables $(X_i^{z})_{i=1 \cdots n(z)}$ drawn from the distribution $X(z)$ where $n(z)$ is the number of times the Markov chain $Z_n$ visited $z$. We define $q^n$ as a solution of
\begin{equation}
\min_{q \in \mathcal{Q}} g_n(q), \label{eq:hatQ_equation}
\end{equation}
with $g_n(q) = \|M^n(q)\| $, and 
$$
M^n(q,z) = \Esp^n[ m(q,X(z),z) ] = \big(\sum_{j=1}^{n(z)} m(q,X_j(z),z)\big)/n(z),
$$
the expected value under the empirical measure $\mu = \big(\sum_{j=1}^{n(z)} \delta_{X_j(z)}\big)/n(z)$ (i.e. empirical risk). We finally define $q^n_k$ as an approximate solution of the problem \eqref{eq:hatQ_equation} returned by an optimization algorithm after $k$ iterations. Thus, we can bound the error $g(q^n_k)$ by
\begin{align*}
0 \leq \Esp\bigg[\bigg(g(q^n_k) - g(q^*) \bigg)(z)\bigg] \leq \underbrace{\Esp\bigg[\bigg(g(q^n) - g(q^*)\bigg)(z)\bigg]}_{\text{estimation error}} + \underbrace{\Esp\bigg[\bigg|g(q^n) - g(q^n_k)\bigg|(z)\bigg]}_{\text{optimization error}} ,
\end{align*}
since $q^*$ minimizes $g$.
\subsection{Convergence rate of the estimation error} 
\manuallabel{subsec:Estim_error}{3.2}
\subsubsection{Slow convergence rate}
We have the following result.
\begin{prop}We assume that the Markov chain $Z_n$ is irreducible. There exists $c_1>0$ such that 
\begin{align*}
 \Esp[ \sup_{q \in \mathcal{Q}} \big|g(q) - g_n(q)\big|] \leq c_1 \frac{1}{\sqrt{n}},\quad \forall z  \in \mathcal{Z}.
\end{align*}
\label{prop:prop_std_unif_bd}
\end{prop}
For sake of completeness, we give the proof of this result in Appendix \ref{app:proof_std_unif_bd}. Proposition \ref{prop:prop_std_unif_bd} allows us to derive the following bound for the estimation error
\begin{align*}
\Esp[\big(g(q^n) - g(q^*)\big)] & = \Esp\big[\big(g(q^n) - g_n(q^n)\big)\big] + \underbrace{\Esp\big[\big(g_n(q^n) - g_n(q^*)\big)\big]}_{\leq 0} + \Esp\big[\big(g_n(q^*) - g(q^*)\big)\big] \\
    & \leq 2 \Esp\big[ \sup_{q} \big|g(q) - g_n	(q)\big|\big] \leq 2 c_1 \frac{1}{\sqrt{n}}. \numberthis \label{Eq:est_error_bound}
\end{align*}
This bound is known to be pessimistic.
\subsubsection{Fast convergence rate}
We obtain the following fast statistical convergence rate.
\begin{prop} Assume that the Markov chain $Z_n$ is irreducible, and
\begin{align*}
\Esp\big[ \sup_{q \in \mathcal{Q}} \big|M(q,z) - M_n(q,z)\big|\,\big|n(z)\big] \leq c' \left(\frac{\log(\bar{n}(z))}{\bar{n}(z)}\right)^{\beta}, \numberthis \label{Eq:fast_bound_cond}
\end{align*}
with $\frac{1}{2} \leq \beta \leq 1$, $c'>0$, and $\bar{n}(z) = n(z) \wedge 1 $. Then, there exists $c_2 > 0$ such that
\begin{align*}
\Esp[\sup_{q \in \mathcal{Q}} \big|g(q^n) - g(q^*)\big|] \leq c_2 \left(\frac{\log(n)}{n}\right)^{\beta},\quad \forall z  \in \mathcal{Z}.
\end{align*}
\label{prop:prop_fast_unif_bd}
\end{prop}
The proof of this proposition is given in Appendix \ref{app:proof_fast_unif_bd}. Since the conclusion of Proposition \ref{app:proof_fast_unif_bd} relies on the condition \eqref{Eq:fast_bound_cond}, we give below two settings under which this condition is fulfilled.

\paragraph{Fast convergence rate for classification problems.}
It is possible to establish \eqref{Eq:fast_bound_cond} for classification problems when
\begin{itemize}
\item The loss function $g$ satisfies regularity conditions, of which the most important are : Lipschitz continuity and convexity, see \cite[Section 4]{bartlett2006convexity}. 
\item The data distribution satisfies some noise conditions, see for instance \cite{tsybakov2004optimal}.
\item The loss function has a bounded moment $\alpha$ with $\alpha > 1$, see \cite[Section 4]{cortes2019relative}.
\end{itemize}
It is also possible to get rid of the $\log(n)$  factor in \eqref{Eq:fast_bound_cond}, see the end of Section 5 in \cite{bousquet2003introduction}.

\paragraph{Fast convergence rate for reinforcement learning.} Since numerical examples focus on RL applications. We propose to derive Inequality \eqref{Eq:fast_bound_cond} for RL problems. To do so we adopt the same notations of Section \ref{sec:rel_RL_SA}.

Let $t$ be the current time. We denote by $A = \{A_t, t<T\}$ the control process of the agent. Since we work under a Markov setting, $A_t$ depends only on $t$, and the current state of the system $U_t$. Moreover, we assume that the agent can take at most a finite set of actions which means $N_A = |\mathcal{A}|<\infty$. For each $A$, we define $q_A$ as follows:
$$
q_A(t,u,a) = \Esp_{A}[\int_{t}^T \rho^{(s-t)}f(s,U_s,A_s) \, ds + \rho^{(T-t)}g(U_T) |U_t =u,A_t = a],
$$
for any $(t,u,a)\in \mathbb{R}_{+} \times \mathcal{U} \times \mathcal{A}$. By definition of $q_A$, we have 
\begin{equation}
q_A(t,u,a) \leq   \Esp [ R_{t+\Delta} + \rho^{\Delta} \underset{a' \in \mathcal{A}}{\sup} \, q_A(t+\Delta,U_{t+\Delta},a')|U_t = u, A_t = a],
\label{ineq:diffusionSubOpti}
\end{equation}
where $R_{t+\Delta}$, and $\rho^{\Delta}$ are defined in \eqref{Eq:Bellman}. Moreover, DPP ensures that the $Q$-function achieves equality in \eqref{ineq:diffusionSubOpti}. Thus, we can define the following loss:
\begin{equation}
g(A) = \sum_{z \in \mathcal{Z}} \Esp[ m(q_A,X(z),z) ],
\end{equation}
for any control process $A$, with $m$, and $X$ introduced in \eqref{Eq:Rel_RL_SA}. Note that this setting is very similar to the classification problem one 
\begin{itemize}
    \item The variable $X$ plays the role of the input variable.
    \item The control process $A$ can be assimilated to the function to learn. Since there are finite numbers of time steps and actions to perform, the strategy $A$ should predict a finite set of options that can be interpreted as labels.
\end{itemize}
Thus, we can try to apply the same techniques and recover similar bounds. Such results are given in the proposition below.

\begin{prop}
Let $B>0$. We define the class of controls $\mathbf{A}_B$ such that
    $$
    \mathbf{A}_B = \{ A; \,\, |m(q_A,x,z)| \leq B, \, \forall (x,z) \in \mathcal{X} \times \mathcal{Z}\}.
    $$
\begin{itemize}
    \item 
    Then, there exists a constant $C>0$ such that with probability at least $1-\delta$,
    $$
    M(q_A,z) - M_n(q_A,z) \leq C \bigg( \cfrac{\log(n(z))}{n(z)} \Var(A,z) + \cfrac{ \log(1/\delta)+ \log(\log(n(z)))}{n(z)}\bigg),
    $$
    with $\Var(A,z)$ the variance of the random variable $m(q_A,X(z),z)$.
    \item If in addition there exists $c>0$, and $\beta \in [0,1]$ such that 
    $$
    \Var(A,z) \leq c M(q_A,z)^\beta, \quad \forall A \in \mathbf{A}_B,
    $$
    then with probability at least $1-\delta$,
    $$
    M(q_A,z) - M_n(q_A,z) \leq C \bigg( \bigg(\cfrac{\log(n(z))}{n(z)}\bigg)^{1/(2-\beta)} + \cfrac{ \log(1/\delta)+ \log(\log(n(z)))}{n(z)}\bigg).
    $$

\end{itemize}
\label{prop:fastConvRateRl}
\end{prop}

The main steps of Proposition's \ref{prop:fastConvRateRl} proof are given in \cite[Theorem 8]{bousquet2003introduction}. It is then standard to derive \eqref{Eq:fast_bound_cond} from Proposition \ref{prop:fastConvRateRl}.

\subsection{Convergence rate of the optimization error}
We turn now to the optimization error. This means that the expected value in \eqref{Eq:Min_pbm} is replaced by the empirical risk, which is known. In such case, one can use many algorithms to find $q_n$. We present in the table below the most important properties of some gradient methods.


{\def\O(#1){O\!\Big(#1\Big)} %
  \def\lO(#1){O\!\Big(\text{\large $\log$}\big(#1\big)\Big)} %
  \def\dlO(#1){O\!\Big(\text{\large $ d \log$}\big(#1\big)\Big)} %
\begin{table}[H]
\resizebox{\linewidth}{!}{
\begin{tabular}{|R{2 cm}|C{3 cm}|C{3 cm}|C{3 cm}|C{3 cm}|C{3 cm}|}
\hline
Algorithm & Cost of one iteration & \multicolumn{2}{|c|}{Iterations to achieve an $\epsilon$ precision }& \multicolumn{2}{|c|}{Time to reach an $\epsilon$ precision} \\
\hline 
 &  & Convex & Strongly convex & Convex & Strongly convex  \\
\hline
GD & $O(\text{\large $ d $}^2)$ & $\O(1/\text{\large$ \epsilon $})$ & {$\lO(1/\text{\large$ \epsilon $})$} & $\O(\text{\large$d$}^2/\text{\large$ \epsilon $})$ & $O\!\Big(\text{\large $ d $}^2\text{\large $ \log $}\big(1/\text{\large$ \epsilon $}\big)\Big)$ \\
\hline 
SGD & $O(\text{\large $ d $})$ & $\O(1/\text{\large$ \epsilon^2 $})$  & $\O(1/\text{\large$ \epsilon $})$   & $\O(1/\text{\large$ \epsilon^2 $})$  & $\O(1/\text{\large$ \epsilon $})$\\
\hline 
Proximal & $O(\text{\large $ d $})$ & $\O(1/\text{\large$ \epsilon $})$  & $\lO(1/\text{\large$ \epsilon $})$   & $\O(1/\text{\large$ \epsilon $})$  & $\dlO(1/\text{\large$ \epsilon $})$\\
\hline 
Acc. prox. & $O(\text{\large $ d $})$ & $\O(1/\text{\large$ \sqrt{\epsilon} $})$  & $\lO(1/\text{\large$ \epsilon $})$   & $\O(\text{\large $ d $}/\sqrt{\epsilon})$  & $\dlO(1/\text{\large$ \epsilon $})$\\
\hline 
SAGA & $O(\text{\large $ d $})$ & $\O(1/\text{\large$ \epsilon $})$  & $\lO(1/\text{\large$ \epsilon $})$   & $\O(1/\text{\large$ \epsilon $})$  & $\dlO(1/\text{\large$ \epsilon $})$\\
\hline 
SVRG & $O(\text{\large $ d $})$ &   & $\lO(1/\text{\large$ \epsilon $})$   &   & $\dlO(1/\text{\large$ \epsilon $})$\\
\hline 
\end{tabular}
}
\caption[Table \ref{Gradientalg}]{Asymptotic properties of some gradient methods.  Note that $d$ is the dimension of the state space $\mathcal{Z}$ and $\epsilon$ is a desired level of accuracy. Here $\epsilon$ corresponds to $1/n$. GD stands for Gradient Descent, SDG for Stochastic Gradient Descent, \emph{Proximal} for Stochastic proximal gradient descent \cite{combettes2011proximal,schmidt2011convergence}, \emph{Acc. prox.} for accelerated proximal stochastic gradient descent \cite{nitanda2014stochastic,schmidt2011convergence}, SAGA for Stochastic accelerated gradient approximation \cite{defazio2014saga}, and SVRG for stochastic variance reduced gradient \cite{johnson2013accelerating}.}
\label{Table:Opti_error_bounds_alg}
\end{table}}
\subsection{Conclusion}
Following the formalism of \cite{bottou2008tradeoffs}, we have decomposed our initial error into
\begin{itemize}
\item \textbf{Estimation error:} its convergence is $O(1/\sqrt{n})$ in pessimistic cases with $n$ the number of iterations. In the other situations, the convergence is faster (i.e. $O\big((\log(n)/n)^{\beta}\big)$) with $1/2 \leq \beta \leq 1$.
\item \textbf{Optimization error:} the convergence is exponential under suitable conditions. In unfavourable cases, the convergence rate is $O(1/n)$.
\end{itemize}
The comparison of these error sources shows that the estimation error is the dominant component. Thus, one can overcome the $O(1/\sqrt{n})$ asymptotic speed, in some situations, by improving the estimation error.

\section{Optimal policy for the learning rate $\gamma$}
\label{sec:optim:rate}

In this section, we take $\mathcal{Z} = \mathbb{N}$ and consider the following type of algorithms:
\begin{align*}
q_{n+1}(Z_n) = q_n (Z_n) - \gamma_n(Z_n) m(q_n,X_{n+1}(Z_n),Z_n), \qquad  \forall n \in \mathbb{N}.
\end{align*}
One can recover the classical SARSA, $Q$-learning, and double $Q$-learning algorithms used in RL by considering a specific expression for $m$ and $X_{n+1}$. In such algorithms the choice of $\gamma_n$ is crucial. One can find in the literature general conditions on $\gamma_n$ needed for convergence of \eqref{Eq:RobinM_RL} such as
\begin{align*}
\sum_{k\geq 0} \gamma_{k}(z) = \infty,\quad a.s,  \qquad \sum_{k\geq 0} \gamma_k^2(z) < \infty,\quad a.s, \quad \forall z  \in \mathcal{Z}.  \numberthis \label{Cond:conv_gamma}
\end{align*}
However, since the set of processes $(\gamma_n)_{n\geq 0}$ satisfying these conditions can be large, and even empty (when $(Z_n)_{n\geq 0}$ is not recurrent), many authors suggest to take $\gamma_n$ proportional to
${1}/{n^{\alpha}}$. The exponent $\alpha$ may vary from
$0$ to $1$ depending on the algorithm used, see \cite{gadat2017optimal,moulines2011non}. Nonetheless, such a choice may be suboptimal. For
example, Figure \ref{Fig:h_2_not_recurrent}.a shows that the blue curve is a way higher than
the orange one. Here, the blue (resp. orange) curve shows how the logarithm of the error varies with $n$ when $\gamma_n = {\eta}/{n}$ (resp. $\gamma_n$ is constant). The constant $\eta$ selected here ensures the fastest convergence for the blue curve.\\

In this paper, we propose to use a stochastic learning rate $(\gamma_k)_{k\geq 0}$; our
learning policy is decomposed into two interacting levels: the inner and the outer level. In
the inner level, we use the \nameref{Alg:v_4_s} algorithm, for ``PAst  Sign Search''. This algorithm
builds a new sequence $(\gamma^i_k)_{k\geq 0}$, based on a predefined sequence
$(\gamma^o_k)_{k\geq 0}$ and the sign variations of $m(q_n,X_{n+1}(Z_n),Z_n)$, whose error
decreases faster than the predefined one. In the outer level, we propose an optimal methodology
for the selection of the predefined sequence $(\gamma^o_k)_{k\geq 0}$. These
two levels are interacting in the sense that the \nameref{Alg:v_4_s} algorithm influences the construction of $(\gamma^o_k)_{k\geq 0}$.
\begin{figure}[H]
\centering
        \includegraphics[width=0.45\linewidth]{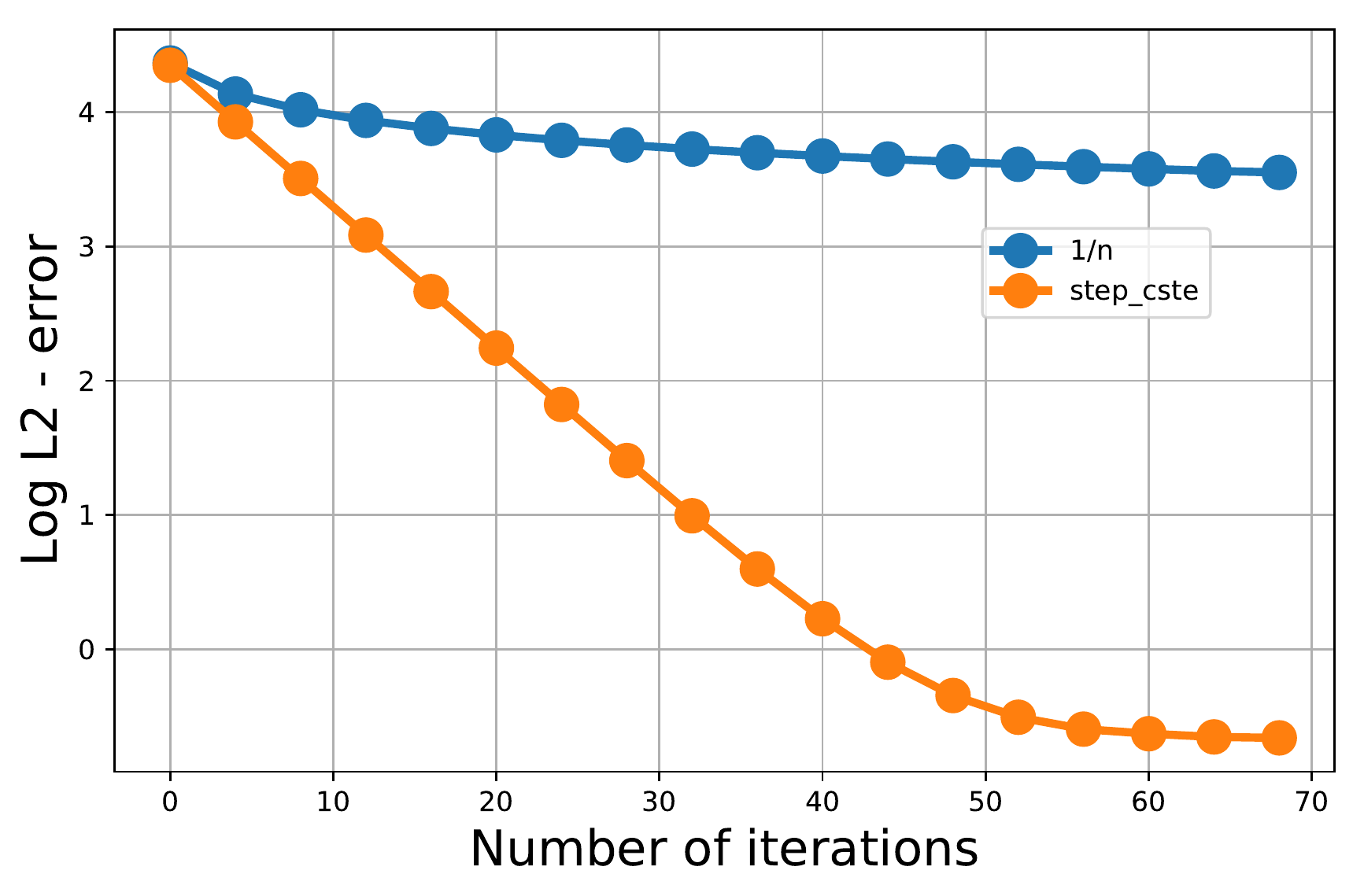}\\
	    \caption{$L^2$-error for the estimation of the drift when $\gamma_k$ is constant in orange and when $\gamma_k \propto \frac{1}{k}$ in blue.}
	     \label{Fig:h_2_not_recurrent}
\end{figure}

\subsection{The inner level}
\subsubsection{The algorithms}
In this part, we introduce three algorithms. We start with our benchmark which is the standard
algorithm used in RL. Then, we present a second algorithm inspired from SAGA
\cite{defazio2014saga}, which is a method used to accelerate the convergence of the stochastic
gradient descent. Under suitable conditions, SAGA has an exponential convergence. Finally, we describe the
\nameref{Alg:v_4_s} algorithm that modifies the learning rate $(\gamma_k)_{k \in
  \mathbb{N}}$ based on the sign variations of $m(q_n,X_{n+1}(Z_n),Z_n)$. The main idea is to
increase $\gamma_n$ as long as the sign of $m(q_n,X_{n+1}(Z_n),Z_n)$
remains unchanged. Then, we reinitialize or lower $\gamma_n$ using a predefined sequence
$(\gamma^o_k)_{k \in \mathbb{N}}$ when the sign of $m(q_n,X_{n+1}(Z_n),Z_n)$ switches. This
algorithm can be seen as an adaptation of the line search strategy, which determines the
maximum distance to move along a given search direction. Actually, the line search method requires a complete knowledge of the cost function because it demands to evaluate several times $g\big(q_k + \gamma
M(q^k)\big) - g\big(q_k\big)$ for different values of $\gamma$, with $g$ being the loss and $M$ representing a proxy of $\nabla g$. However, our approach has neither
access to $g$ nor $M$. It can only compute $m(q_n,X_{n+1}(Z_n),Z_n)$ when the state $z = Z_k$
is visited. Moreover, to get a new observation it needs to wait\footnote{This waiting time may
  be very long depending on the dimension of the state space $\mathcal{Z}$ and the properties
  of the process $(Z_k)_{k \geq 0}$.} for the next visit of the state $z = Z_k$. Nevertheless, it has instantaneous access to previously observed values. Thus, the main idea here is to use these past observations. Some theoretical properties of these algorithms are investigated in Section \ref{subsec:Main_res}. 
\begin{alg}[RL]We start with an arbitrary $q_0 \in \mathcal{Q}$ and define by induction $q_k$\footnote{\label{note1}Non-visited coordinates are not modified. For example, we set $q_{k+1}(z) = q_{k}(z)$ for all $z \ne Z_k$.} as follows:
\begin{align*}
\begin{array}{lcl}
q_{k+1}(Z_k) & = &  q_{k}(Z_k) - \gamma_k(Z_k) m(q_k,X_{k+1}(Z_k),Z_{k}).
\end{array}
\end{align*}

\label{Alg:v_1_s}
\end{alg}
\begin{alg}[SAGA] We start with an arbitrary $q_0 \in \mathcal{Q}$, $M_0 = 0$\footnote{Here $M_0$ is the zero function in the sense that $M_0[z,i] = 0$ for any $z \in \mathcal{Z}$ and $i \in \{1,\ldots,M\}$.} and define by induction $q_k$ and $M_k$\footnoteref{note1} as follows:
\begin{align*}
\begin{array}{lcl}
q_{k+1}(Z_k) & = & q_{k}(Z_k) - \gamma_k(Z_k) \left[m(q_k,X_{k+1}(Z_k),Z_{k}) - M_k[Z_k,\color{blue}i\color{black}] + \color{blue} \cfrac{\big(\sum_{j=1}^M M_k[Z_k,j] \big)}{ M}\color{black}\right] ,\\
M_{k+1}[Z_k,i]  & = & m(q_k,X_{k+1}(Z_k),Z_{k}),
\end{array}
\end{align*}
with $i$ picked from the distribution $p =  (\sum_{i=1}^M \delta_i)/M $.
\label{Alg:v_2_s}
\end{alg}
For the next algorithm, we give ourselves a predefined learning rate called $(\gamma_k)_{k \geq 0}$, a function $h:\mathbb{R}_+ \times \mathbb{R}_+  \rightarrow \mathbb{R}_+$ to increase the current learning rate, and another one
$l:\mathbb{R}_+\times \mathbb{R}_+ \rightarrow \mathbb{R}_+$ to lower it. The function $h$ is used to accelerate the descent, while the function $l$ goes back to a slower pace.
\begin{alg}[PASS] We start with an arbitrary $q_0$ and define by induction $q_k$ and $\hat{\gamma}_k$\footnoteref{note1} as follows:
\begin{itemize}
\item If $m(q_n,X_{n+1}(Z_n),Z_n) \times m(q_{r^n_1},X_{r^n_1+1}(,Z_{r^n_1}),Z_{r^n_1}) \geq 0$, then do 
\begin{align*}
\begin{array}{lcl}
q_{n+1}(Z_n) & = &  q_{n}(Z_n) - h\big(\hat{\gamma}_n(Z_n),\gamma_n(Z_n)\big) m(q_n,X_{n+1}(Z_{n}),Z_{n}),\\
\hat{\gamma}_{n+1}(Z_n) & = & h\big(\hat{\gamma}_n(Z_n),\gamma_n(Z_n)\big), 
\end{array}
\end{align*}
with $r^n_1$ is the index of the last observation when the process $Z$ visits the state $Z_n$.
\item Else, do 
\begin{align*}
\begin{array}{lcl}
q_{n+1}(Z_n) & = &  q_{n}(Z_n) - l\big(\hat{\gamma}_n(Z_n),\gamma_n(Z_n)\big) m(q_n,X_{n+1}(Z_{n}),Z_{n}),\\
\hat{\gamma}_{n+1}(Z_n) & = & l\big(\hat{\gamma}_n(Z_n),\gamma_n(Z_n)\big).
\end{array}
\end{align*}
\end{itemize}
\label{Alg:v_4_s}
\end{alg}
\subsubsection{Assumptions}
In this section, we present the assumptions needed to study the convergence of Algorithms \nameref{Alg:v_1_s}, \nameref{Alg:v_2_s} and \nameref{Alg:v_4_s}. We assume that Assumption \ref{assump:assump__0} is in force. Hence, there exists $q^*$, a solution of \eqref{Eq:init_pbm}. We write $m^*$ for the vector $m^*(x,z) = m(q^*,x,z),\, \forall (x,z) \in \mathcal{X} \times \mathcal{Z}$. Recall that $ \Esp[m^*(X(z),z)]= 0,\, \forall z \in \mathcal{Z}$. Let us consider the following assumptions:

\begin{Assumption}[Pseudo strong convexity 2] There exists a constant $L > 0$ such that

\begin{align*}
\big(\Esp_k[m(q_k,X_{k+1}(Z_{k}),Z_{k})]\big)\big(q_k(Z_{k}) - q^*(Z_{k})\big) \geq L \big(q_k(Z_{k}) - q^*(Z_{k})\big)^2, 
\end{align*}

with $\Esp_k[X] = \Esp[X|\mathcal{F}_k]$ for any random variable $X$.
\label{assump:assump__1_s}
\end{Assumption}

Note that Assumption \ref{assump:assump__1_s} is natural in the deterministic framework. For instance, if we take a strongly convex function $f$ and call $m$ its gradient (i.e $m = \nabla f$). Then, $m$ satisfies Assumption \ref{assump:assump__1_s}. Additionally, the pseudo-gradient property (PG) considered in \cite[Section 4.2]{bertsekas1996neuro} is close to Assumption \ref{assump:assump__1_s}. However, Assumption \ref{assump:assump__1_s} is slightly more general than PG since it involves only the component's norm $(q_k - q^*)(Z_k)$ instead of the vector's $(q_k - q^*)$. To get tighter approximations, we also introduce the quantity $L_k$ as follows:

\begin{equation*}
L_k = \left\{
\begin{array}{ll}
\cfrac{\Esp_k[m(q_k,X_{k+1}(Z_{k}),Z_{k})]}{q^k(Z_{k}) - q^*(Z_{k})},& \text{ If } q^k(Z_{k}) - q^*(Z_{k}) \ne 0,\\
0, & \text{ otherwise.}
\end{array}
\right.
\end{equation*}

Note that $L_k \geq 0$ under Assumption \ref{assump:assump__1_s}. It is also the biggest constant that satisfies Assumption \ref{assump:assump__1_s} for a fixed $k$. In particular, this means that $L_k \geq L$.

\begin{Assumption}[Lipschitz continuity of $m$] There exists a positive constant $B>0$ such that for any random variables $X$ and $X'$ valued in $\mathcal{X} $ we have

\begin{align*}
\Esp_k\big[\big(m(q_k,X,Z_{k}) - m^*(X',Z_{k})\big)^2\big] \leq B \big\{1 + \big({q_k}(Z_{k}) - q^*(Z_{k})\big)^2 + \Esp_k\big[\big(X - X'\big)^2\big]\big\},
\end{align*}

with $\Esp_k[X] = \Esp[X|\mathcal{F}_k]$ for any random variable $X$.
\label{assump:assump__2_s}
\end{Assumption}

Assumption \ref{assump:assump__2_s} guarantees that $m$ is Lipschitz. Authors in \cite[Section 4.2]{bertsekas1996neuro} use a similar condition. To get better bounds, we introduce $B_k$ such that
\begin{equation*}
B_k = \cfrac{\Esp_k\big[\big(m(q_k,X,Z_{k}) - m^*(X',Z_{k})\big)^2\big]}{1 + \big({q_k}(Z_{k}) - q^*(Z_{k})\big)^2 + \Esp_k\big[\big(X - X'\big)^2\big]}.
\end{equation*}
We have $B_k \leq B$ since $B_k$ is the smallest constant satisfying Assumption \ref{assump:assump__2_s} for a fixed $k$. 
We finally add an assumption on the learning $(\gamma_k)_{k \geq 0}$.
\begin{Assumption}[Learning rate explosion]For any $z\in \mathcal{Z}$, we have
\begin{align*}
\sum_{ k \geq 1} \gamma_k(z) = \infty,\quad a.s.
\end{align*} 
\label{assump:assump__3_s}
\end{Assumption}
\vspace{-0.5cm}
When the process $Z$ is Markov and $\gamma_k(z)$ bounded, Assumption \ref{assump:assump__3_s} ensures that $Z$ is recurrent. To see this, we first assume that $\gamma_k(z)$ is uniformly bounded without loss of generality. In such a case, there exists $A$ such that $\gamma_k(z) \leq A,$ for all $k \geq 1$. Thus, we get $\sum_{k \geq 1} \Esp[\gamma_k(z) \mathbf{1}_{Z_k = z}] \leq  A \sum_{k \geq 1}\mathbb{P}[Z_k = z] $. Since the left hand side of the previous inequality diverges under Assumption \ref{assump:assump__3_s}, we have 
$$
\sum_{k \geq 1} \mathbb{P}[Z_k = z] = \infty,
$$
which proves that $Z$ is recurrent. 

\subsubsection{Main results}
\manuallabel{subsec:Main_res}{4.3}
In this section, we compare Algorithms \nameref{Alg:v_1_s}, \nameref{Alg:v_2_s}, and \nameref{Alg:v_4_s} and prove the convergence of \nameref{Alg:v_4_s}. Let $c$ be a positive constant and $k \in \mathbb{N}$. We define the error function as follows:
\begin{align*}
e^k(z) = \left\{
\begin{array}{ll}
(q_{k}(z) - q^*(z))^2 , & \text{ for Algorithms \nameref{Alg:v_1_s}, and \nameref{Alg:v_4_s},}\\
\cfrac{\sum_{j=1}^M \left(M^k[z,j] - m^*(z) \right)^2}{M}  + c (q_{k}(z) - q^*(z))^2, & \text{ for Algorithm \nameref{Alg:v_2_s},}
\end{array}
\right.
\end{align*}
for all $z \in \mathcal{Z}$ and $j \in \{1,\ldots,M\}$. We write $E^k$ for the total error $E^k = \|e^k\|_{\nu} = \sum_{z \in \mathcal{Z}} e^k(z) \nu_z$ with $(\nu_z)_{z \in \mathcal{Z}}$ a non-negative sequence.\footnote{The sequence $(\nu_z)_{z \in \mathcal{Z}}$ is used to ensure that the error is bounded when needed.} We also use the following notations: 
\begin{align*}
p(x) = 2 L x  - B x^2, \qquad p_k(x) = 2 L_k x  - B_k x^2, \qquad \bar{\gamma}_k = \arg\sup_{l \in \mathbb{R}} p_k(l) = \cfrac{L_k}{B_k}, \qquad  \forall x \in \mathbb{R}.
\end{align*}
\begin{prop}Let $z \in \mathcal{Z} $. Under Assumptions \ref{assump:assump__0}, \ref{assump:assump__1_s}, \ref{assump:assump__2_s}, and when there exists $r_1 \geq 1$ such that 
$$
\gamma_2 \leq l(\gamma_1,\gamma_2), \qquad \text{ and } \qquad h(\gamma_1,\gamma_2) \leq r_1 \gamma_2, \qquad \forall (\gamma_1,\gamma_2) \in \mathbb{R}_+^2,
$$
we have
\begin{align*}
\mathbf{1}_{A}\Esp_k[e^{k+1}(z)]  & \leq \mathbf{1}_{A} \big[\alpha_k e^{k}(z) + M_k\big], \numberthis \label{Eq:prop_0_error}
\end{align*}
with $A = \{Z_k = z\}$. The constants $\alpha_k$ and $M_k$ vary from one algorithm to another as follows:
\begin{align*}
\alpha_k(z_1) = \left\{
\begin{array}{ll}
\big[1- p\big(\gamma_k(z)\big)\big],& \text{ for Algorithm \nameref{Alg:v_1_s},}\\
\max \bigg(1- \big(2 L \gamma_k(z)  - 3B \gamma_k(z)^2\big) + \frac{B}{Mc} , 1 - \big(\frac{1}{M} - 6 \gamma^2_k(z) c\big)\bigg),& \text{ for Algorithm \nameref{Alg:v_2_s},}\\
\big\{1- p_k\big(\underline{\gamma}_k\big) + d^1 \gamma_k^2(z) \mathbf{1}_{c_k \geq 1}\big\},& \text{ for Algorithm \nameref{Alg:v_4_s},}\\
\end{array}
\right.
\numberthis \label{Eq:val_eq_alpha}
\end{align*} 
and 
\begin{align*}
M_k = \left\{
\begin{array}{ll}
B \gamma^2_k(z) ( 4 + 3v_k),& \text{ for Algorithm \nameref{Alg:v_1_s},}\\
3B \gamma^2_k(z) ( 4 + 3v_k),& \text{ for Algorithm \nameref{Alg:v_2_s},}\\
B \big(c_k\bar{\gamma}_k\big)^2 ( 4 + 3v_k),& \text{ for Algorithm \nameref{Alg:v_4_s},}\\
\end{array}
\right.
\numberthis \label{Eq:val_eq_M}
\end{align*} 
with $c_k = \cfrac{\Esp_k[\hat{\gamma}_k(z) m(q^k,X_{k+1}(z),z)]}{\bar{\gamma}_k(z)\Esp_k[m(q^k,X_{k+1}(z),z)]}$, $\underline{\gamma}_k = c_k(z) \bar{\gamma}_k \vee \bar{\gamma}_k$, $d^1 =  (r_1-1)^2 B_k$ and $v_k = \Var(Z_k)$.
\label{prop:prop_0}
\end{prop}
The proof of Proposition \ref{prop:prop_0} is given in Appendix \ref{app:proof_prop_0}. Equation \eqref{Eq:prop_0_error} reveals that the performance of Algorithms \nameref{Alg:v_1_s}, \nameref{Alg:v_2_s}, and \nameref{Alg:v_4_s} depends on the interaction between two competing terms:
\begin{itemize}
\item On the one hand the slope $\alpha_k$ controls the decrease of the error from one step to the next.
\item On the other hand the quantity $M_k$ gathers two sources of imprecision: the estimation and optimization errors. Both sources of imprecision have a variance term $v_n$ (because the distribution of $Z$ is unknown), and a positive constant (coming from the noisy nature of observations).
\end{itemize}
There is a competition between these two terms: to decrease $M_k$ we need to send $\gamma_k$ towards zero while the reduction of $\alpha_k$ requires a relatively small but still non-zero $\gamma_k$. Thus, $\gamma_k$ should satisfy a trade-off in order to ensure the convergence of the algorithms. The RM conditions \eqref{Cond:conv_gamma} are a way to address this trade-off. Now, in order to analyse the properties of each  algorithm, we compare for a fixed $\gamma_k$ its respective values of $\alpha_k$, and $M_k$ in Table \ref{Table:Opti_compare_alg}. For sake of clarity, we choose to present the variable $(1 -\alpha_k)$ instead of $\alpha_k$ in this table; note that a large value of $1 -\alpha_k$ means that $\alpha_k$ is small and thus induces a fast convergence.
\begin{table}[H]
\resizebox{\linewidth}{!}{
\begin{tabular}{|R{2.5 cm}|C{1 cm}C{0.5 cm} R{6 cm}|C{2.4 cm}|C{3 cm}|C{2.4 cm}|}
\hline
 \large Algorithms &  \multicolumn{4}{c}{$(1 - \text{\large $\alpha$}_k)$ }& \multicolumn{2}{|c|}{$\text{\large $M$}_k$} \\
\hline 
 &  \multicolumn{3}{c|}{\large Value } & \large Comparison with Algo 1 & \large Value & \large Comparison with Algo 1   \\
 \hline 
RL (Algo 1)  & $2\text{\large $\gamma$}_k(z_1) \text{\large $L$}$ & $-$ & $\text{\large $B$} \text{\large $\gamma$}^2_k(z_1)$ &  --- & $\text{\large $B$} \text{\large $\gamma$}^2_k ( 4 + 3 \text{\large $v$}_k)$ &   --- \\
\hline 
SAGA (Algo 2)  & \multicolumn{3}{l|}{
\hspace{-0.1cm}$\bigg(2\text{\large $\gamma$}_k(z_1) \text{\large $L$} - 3 \text{\large $B$} \text{\large $\gamma$}^2_k(z_1) \color{blue}- \frac{\text{\large $B$}}{\text{\large $M$}\text{\large $c$}}\color{black}\bigg) \vee \color{blue}\bigg(\frac{1}{\text{\large $M$}} - 6 \text{\large $\gamma$}^2_k(z_1) \text{\large $c$}\bigg)\color{black}$ } & smaller   & $3\text{\large $B$} \text{\large $\gamma$}^2_k ( 4 + 3\text{\large $v$}_k)$  & larger\\
\hline 
PASS (Algo 3)  & $2\underline{\text{\large $\gamma$}}_k \text{\large $L$}_k$ & $-$ & $\text{\large $B$}_k \big(\underline{\text{\large $\gamma$}}_k\big)^2 +\color{blue}\text{\large $d$}^1 \text{\large $\gamma$}_k^2 \mathbf{1}_{\text{\large $c$}_k \geq 1}\color{black}$  &  larger   & $\text{\large $B$} \big(\text{\large $c$}_k\bar{\text{\large $\gamma$}}_k\big)^2( 4 + 3\text{\large $v$}_k)$  &  larger\\
\hline 
\end{tabular}}
\caption[Table \ref{Gradientalg}]{Comparison of the algorithms \nameref{Alg:v_1_s}, \nameref{Alg:v_2_s} and \nameref{Alg:v_4_s}.}
\label{Table:Opti_compare_alg}
\end{table}

Let $n \in \mathbb{N}^{*}$, $j \leq n$, $z_1 \in \mathcal{Z}$, and $\tau_{z_1} = \inf \{l>0,\, Z_{l} = z_1\}$. We need to introduce the following notations: $a_{j}= \mathbb{P}[\tau_{z_1} \geq j | Z_{0} = z_1]$, $b_j= \mu_j =  \frac{\Esp_0[\alpha_j e^{j}(z_1)] }{ \Esp_0[e^{j}(z_1)]}$, $r_j = 1 - \mu_j$, $\bar{\mu}^n_{j} = e^{- \sum_{j=n-j+1}^n r_j}$, $\bar{a}_{j} = a_j / r$, and $r = \sum_{j \geq 1} a_{j} $. Finally, we write $\bar{a}^{*\infty}_n = \lim_{m \rightarrow \infty} \bar{a}^{*m}_n $, and define the sequence $(\bar{a}^{* m}_k)_{k \geq 1}$ recursively such that $\bar{a}^{* 1}_k = \bar{a}_k$ and $\bar{a}^{* (m+1)}_k = \sum_{l = 1}^k \bar{a}_{k+1-l} \bar{a}^{* m}_l$ for all $k \geq 1$ and $m \geq 1$. The result below holds only for Algorithms \ref{Alg:v_1_s} and \ref{Alg:v_4_s}.
\begin{theo} Let the Assumptions \ref{assump:assump__0}, \ref{assump:assump__1_s}, \ref{assump:assump__2_s} and \ref{assump:assump__3_s} be in force. Then, Algorithms \ref{Alg:v_1_s} and \ref{Alg:v_4_s} verify
\begin{itemize}
\item When $\sum_{k \geq 0} \text{\large $\gamma$}_k^2(z) < \infty$ for all $z \in \mathbb{Z}$, we have 
\begin{align*}
E^n \underset{n \rightarrow \infty}{\rightarrow} 0.\numberthis \label{eq:conv_error_numscheme_1_1}
\end{align*}
in probability.
\item If in addition $(Z_n)_{n \geq 1}$ is an homogeneous Markov chain, and $\Esp[E^1] < \infty$, then 
\begin{align*}
\Esp[E^n] \underset{n \rightarrow \infty}{\rightarrow} 0.\numberthis \label{eq:conv_error_numscheme_1_11}
\end{align*}
\item Moreover, under the same condition, there exists a constant $B'\geq 0$ such that 
\begin{equation}
\Esp_0[e^n(z_1)]  \leq B' \sum_{(l,j,i)\in (\mathbb{N}^*)^3, l+j+i = n}\bar{\epsilon}_j \bar{\mu}_l^j \bar{a}^{* \infty}_{i},\label{eq:conv_error_numscheme_1_2}
\end{equation}
with $\bar{\epsilon} = b_n \epsilon$, and $\epsilon_n = e^1(z_1)a_n(z_1) + \sum_{j=1}^{n-1} a_{n-j}  \Esp[M_j]$.
\end{itemize}
\label{theo:conv_error_numscheme_1}
\end{theo}


Equations \eqref{eq:conv_error_numscheme_1_1}, and \eqref{eq:conv_error_numscheme_1_11} ensure the convergence of the error $E^n$ towards zero in both probability and $L^2$. Equation \eqref{eq:conv_error_numscheme_1_2} gives an upper bound for the error. In particular, it shows how the terms $\bar{\epsilon}_n$, $\bar{\mu}^n_k$, and $\bar{a}^{* \infty}_n$ interact together to decrease of the error $\Esp_0[e^n(z_1)]$.  We recall that $\bar{\epsilon}_n$ is a noise term that gathers the sources of imprecision, $\bar{\mu}$ represents the influence of the slope factor $\alpha$, $\bar{a}^{* \infty}_n$ is related to the probability distribution of the process $(Z_k)_{k \geq 0}$. Note that the right-hand side of \eqref{eq:conv_error_numscheme_1_2} is not trivial and it converges towards $0$ when there exists $L'\geq 0$ such that $\Esp_0[ \alpha_j e^{j}(z_1)] \leq L' \Esp_0[e^{j}(z_1)] \Esp[\alpha_j]$ for all $ j \geq 1$.

\subsection{The upper level}
\manuallabel{subsec:Opti_policy}{4.2}
In practice, to apply \nameref{Alg:v_4_s} we need an appropriate predefined sequence $(\gamma_k)_{k \in \mathbb{N}}$. It is possible to take $\gamma_k$ proportional to ${1}/{k^{\alpha}}$ with $\alpha \in  (0,1]$ as proposed in \cite{moulines2011non,robbins1951stochastic}. However, in this section, we present an optimal dynamic policy for the choice of the learning rate $(\gamma_k)_{k \in \mathbb{N}}$.  To do so, we assume that 

\begin{align*}
e^{n+1}(z) = \mathbf{1}_{A}\big( \alpha_n e^{n}(z) + M_n + S_n \big) + \mathbf{1}_{A^c} e^{n}(z),\numberthis \label{eq:dynamic_error_e}
\end{align*}

with $A = \{Z_n = z\}$, $S_n$ does not depend on the learning rate and verifies $\Esp[\mathbf{1}_{A}S_n ]\leq 0$. Equation \eqref{eq:dynamic_error_e} is consistent with Proposition \ref{prop:prop_0}. Moreover, we force $e^n$ to stay below the upper bound $x_2 = \arg\sup_{x \in \mathbb{R}_{+}} g(x)$ with 
$$
g(x) =  x - \cfrac{L^2 x^2}{2B(x + (2 + v))}, \forall x \in \mathbb{R}.
$$
Such constraint is not that restrictive since we know that the error $e^{n}$ converges towards $0$. Since $\alpha_n$, and $M_n$ are both functions of the learning rate, the idea is to choose the learning rate $\gamma_n$ such that

\begin{align*}
\gamma_n = \argmin_{\gamma} \big( \alpha_n (\gamma)  e^{n}(z) + M_n(\gamma) \big),\numberthis \label{eq:dynamic_gamma_opti}
\end{align*}

which gives 
\begin{align*}
\gamma_n = \left\{
\begin{array}{ll}
\cfrac{L}{B} \cfrac{ e^{n}(z)}{e^n(z) + (2 + v_n)},& \text{ for Algorithm \nameref{Alg:v_1_s},}\\
\cfrac{L_n}{B_n} \cfrac{e^{n}(z)}{e^n(z) + d_1/B_n \mathbf{1}_{c_n \geq 1} +  c_n^2(2 + v_n)},& \text{ for Algorithm \nameref{Alg:v_4_s},}\\
\end{array}
\right.
\end{align*}

The constants $L$, $L_n$, $B$, $B_n$, $d_1$ and $c_n$ are defined in Proposition \ref{prop:prop_0}. Note that, the value $L/B$ is known to be a good choice for the learning rate. Thus, the proposed $\gamma_n$ introduces a variation around this value that takes into account both the variance $v_n$, and an estimate of the past observed error $e^n$. The algorithm PASS adds a supplementary optimization layer since the global constants $L$, and $B$ are replaced by the more local ones $L_n$, and $B_n$.  

We write $\Gamma$ for the set of processes $\tilde{\gamma} = (\tilde{\gamma}_n)_{n \geq 0}$ adapted to the filtration generated by the observed errors $(e^n)_{n \geq 0}$. We have the following result. 

\begin{prop} The sequence $\gamma = (\gamma_n)_{n \geq 0}$ defined in \eqref{eq:dynamic_gamma_opti} satisfies 

$$
e_{\gamma}^n \leq e_{\tilde{\gamma}}^n, \quad a.s, \quad \forall n \in \mathbb{N},\, \forall \tilde{\gamma} \in \Gamma.
$$

We write $e_{\gamma}^n$ to point out the dependence of the error $e^n$ on the chosen control $\gamma$.
\label{prop:opti_gamma}
\end{prop}

The proof of the above result is given in Appendix \ref{sec:proof_speed_conv}. Proposition \ref{prop:opti_gamma} shows that $\gamma$ ensures that fastest convergence speed of the error and 

$$
\Esp[e_{\gamma}^n] = \inf_{\tilde{\gamma} \in \Gamma } \Esp[e_{\tilde{\gamma}}^n], \quad \forall n \in \mathbb{N}.
$$

This guarantees its optimality.\\

We end this section with some practical considerations. Note that $e^n$ is not known in practice because we do not have access to $q^*$. However, one can take the average value of $m(q_n,X_{n+1}(Z_n),Z_n) $ over the last $p \in \mathbb{N}^*$ visit times as a proxy of $e^n(Z_n)$. Moreover, the constants $L$ and $B$ are also unknown in practice. To tackle this issue, a first solution consists of starting with arbitrary values for $B$ and $L$ and generating a sequence of learning rates. If the error $m(q_n,X_{n+1}(Z_n),Z_n)$ increases, we take a larger value for $B$ and a smaller one for $L$ otherwise $B$ and $L$ values are kept unchanged. Finally, an alternative solution for the choice of the upper level learning rate consists of considering a piece-wise constant (PC) policy. To do so, one can track the average error of $m(q_n,X_{n+1}(Z_n),Z_n)$ over the lasts $p$ visit times. If this average error does not decrease, the step size is divided by a factor $\alpha$. 

\subsection{Extension}
The results of this section still hold when the descent sequence \nameref{Alg:v_4_s} is replaced by the vectorial version
\begin{itemize}
\item If $ \langle \gamma_n m(q_n,X_{n+1}) , m(q_{n-1},X_{n}) \rangle \geq 0$, then do 
\begin{align*}
\begin{array}{lcl}
q_{n+1} & = &  q_{n} - h\big(\hat{\gamma}_n,\gamma_n\big) m(q_n,X_{n+1}),\\
\hat{\gamma}_{n+1} & = & h\big(\hat{\gamma}_n,\gamma_n\big), 
\end{array}
\end{align*}
with $m(q,X)$, and $h\big(\hat{\gamma}_n,\gamma_n\big)$ respectively the vectors $m(q,X)(z) = m(q,X(z),z)$, and $h\big(\hat{\gamma}_n,\gamma_n\big)(z) = h\big(\hat{\gamma}_n(z),\gamma_n(z)\big)$ for any $z \in \mathcal{Z}$, $q \in \mathcal{Q}$, and $X \in \mathcal{X}^{\mathcal{Z}}$. 
\item Else, do 
\begin{align*}
\begin{array}{lcl}
q_{n+1} & = &  q_{n} - l\big(\hat{\gamma}_n,\gamma_n\big) m(q_n,X_{n+1}),\\
\hat{\gamma}_{n+1} & = & l\big(\hat{\gamma}_n,\gamma_n\big),
\end{array}
\end{align*}
with $l\big(\hat{\gamma}_n,\gamma_n\big)$ the vector $l\big(\hat{\gamma}_n,\gamma_n\big)(z) = h\big(\hat{\gamma}_n(z),\gamma_n(z)\big)$ for any $z \in \mathcal{Z}$. 
\end{itemize}

When $\gamma_n(z) = 0$ if $z \ne Z_n$, we recover the standard PASS algorithm. Thus, the vectorial version is slightly more general and uses the scalar product between vectors instead of the product between two coordinates. 

\section{Some examples}
\label{sec:examples}
\subsection{Methodology}
\label{subsec:method_num}

The code and numerical results presented in this section can be found in \url{https://github.com/othmaneM/RL_adap_stepsize}. Here, we compare four algorithms. The two first ones are two different versions of \nameref{Alg:v_1_s}. In the first version, the learning rate $\gamma_k(z)$ is taken such that $\gamma_k = \eta/n_k(z)$ with $\eta >0$ selected to provide the best convergence results and $n_k(z)$ the number of visits to the state $z$. In the second version, the step size follows the piece-wise constant policy (PC) described at the end of Section \ref{subsec:Opti_policy}. The third algorithm is \nameref{Alg:v_2_s} where the step size is derived from PC policy. Finally, we use the \nameref{Alg:v_4_s} algorithm presented in the previous sections with a predefined learning rate following the PC policy. We consider three numerical examples to compare the convergence speed of these algorithms: drift estimation, optimal placement of limit orders and the optimal liquidation of shares.	
	
\subsection{Drift estimation}
\label{sec:drift_estim}

\paragraph{Formulation of the problem.} We observe a process $(S_{n})_{n\geq 0}$ which satisfies 
\begin{align*}
S_{n+1} = S_n + f_{n+1} + W_n, \numberthis \label{Eq:S_dyn}
\end{align*}
with $W_n$ a centred noise with finite variance. We want to estimate the quantities $f_i$ with $i \in  \{1,\cdots,n_{max} \}$. Using \eqref{Eq:S_dyn} and $\Esp[W_t] = 0$, we get 
\begin{align*}
\Esp\big[S_{i+1} - S_{i} - f_{i+1} \big] = 0, \qquad \forall i \in \{0,\cdots,n_{max}-1 \}.
\end{align*}
Thus, we can estimate $f_i$ using stochastic iterative algorithms. The pseudo-code of our implementation of \nameref{Alg:v_4_s}  for this problem can be found in the Appendix \ref{sec:implementations_} under the name Implementation \ref{Alg:PASS_drift_estim}.

\paragraph{Numerical results.} Figure \ref{Fig:h_2} shows the variation of the $L^2$-error when the number of iterations increases. We can see that the algorithm \nameref{Alg:v_4_s} outperforms standard stochastic approximation algorithms. Moreover, other algorithms behave as expected: the standard RL decreases very slowly (but we know it will drive the asymptotic error to zero), the constant learning rate and SAGA provides better results than RL, while \nameref{Alg:v_4_s} seems to have captured the best of the two worlds for this application: very fast acceleration at the beginning and the asymptotic error goes to zero.
\begin{figure}[H]
  \centering
  \hfill $L^2$-error against the number of iterations \hfill ~\\
  \includegraphics[width=0.65\linewidth]{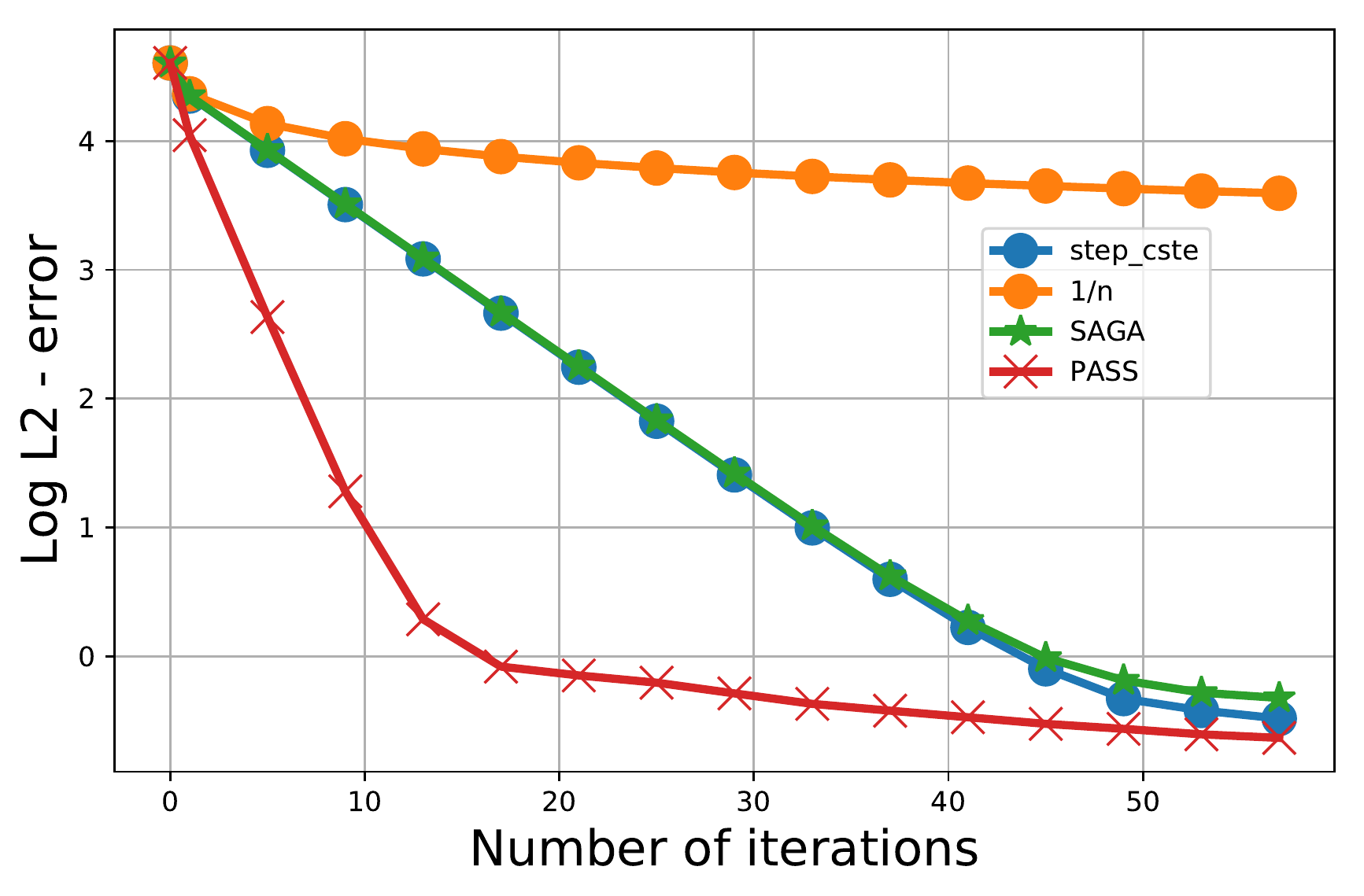}\\
  \caption{The $L^2$-error between $f^{k}$ and $f$ for different numerical methods averaged over 1000 simulated paths.}
  \label{Fig:h_2}
\end{figure}

\subsection{Optimal placement of a limit order}
\label{sec:optim:limit}

\paragraph{Formalisation of the problem.} We consider an agent who aims at buying a unit quantity using limit orders, and market orders during the time interval $[0,T]$ (see \cite{lehalleLimitOrderStrategic2017} for detailed explanations). In such case, the agent wonder how to find the right balance between fast execution and avoiding trading costs associated to the bid-ask spread. The agent state at time $t$ is modelled by $X_t = (Q^{Before},Q^{After},P)$ with $Q^{Before}$ the number of shares placed before the agent's order, $Q^{After}$ the queue size after the agent's order, and $P_t$ the mid price, see Figure \ref{fig:lob:flows-opti_placement}. The agents wants to minimise the quantity 
$$
\Esp[F(X_{T \wedge T^{\textrm{exec}} \wedge  \tau }) + \int_0^{T \wedge T^{\textrm{exec}} \wedge  \tau} c \, ds],
$$
where 
\begin{itemize}
\item $T$ is the final time horizon.
\item $T^{\textrm{exec}} = \inf\{ t \geq 0, \, P_t = 0\}$ is the first time when the limit order gets a transaction.
\item $\tau$ is the first time when a market order is sent.
\item $X=(Q^{Before},Q^{After}, P)$ is the state of the order book.
\item $F(u)$ is the price of the transaction (i.e. $F(u)=p+\psi$ when the agents crosses the spread and $F(u)=p$ otherwise).
\end{itemize}

We show in Section \ref{sec:rel_RL_SA} that the $Q$-function is solution of \eqref{Eq:Rel_RL_SA}, see details in \url{https://github.com/othmaneM/RL_adap_stepsize}. Thus, we can use Algorithms \nameref{Alg:v_1_s}, \nameref{Alg:v_2_s} and \nameref{Alg:v_4_s} to estimate it. The pseudo-code of our implementation of \nameref{Alg:v_4_s} is available as Implementation \ref{Alg:PASS_opti_placement} in Appendix \ref{sec:implementations_}.
\begin{figure}[ht!]
\begin{center}
\begin{tikzpicture}
\draw (3,-1) rectangle (4,-5);
\draw [black,fill=gray!60] (3,-3) rectangle (4,-3.3);
\draw (6,-2) rectangle (7,-5);
\draw [dotted][->](0,-5) -- ++(10,0);
\draw (5,-5) node[sloped]{$|$};
\draw (3.5,-5) node[below]{$ \text{Same} $} ;
\draw (6.5,-5) node[below]{$ \text{Opp} $} ;
\draw (3,-4) node[left]{$Q^{Before}$} ;
\draw (3,-3.15) node[left]{$ \text{q} $} ;
\draw (3,-2) node[left]{$Q^{After}$} ;
\draw (7,-4) node[right]{$Q^{Opp}$} ;
\footnotesize
\draw (5,-5) node[below]{$ P(t) $} ;
\normalsize
\draw (10,-5) node[right]{$Price$} ;
\end{tikzpicture}
\end{center}
    \caption{The state space of our limit order control problem.}
  \label{fig:lob:flows-opti_placement}
\end{figure}
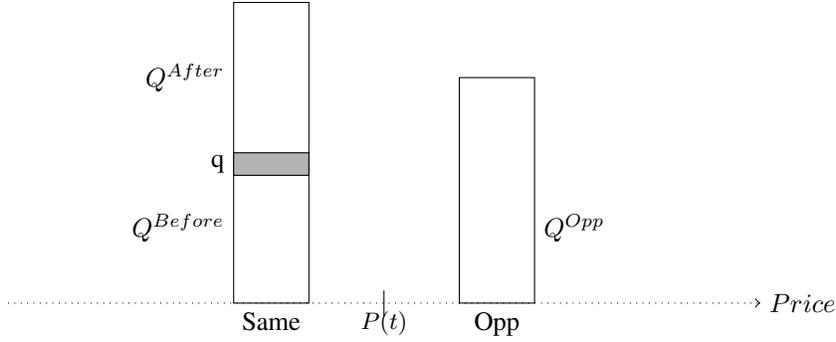
\paragraph{Numerical results.}Figure \ref{Fig:h_2_2_intro} shows three control maps: the x-axis reads the quantity on ``same side'' (i.e. $Q^{same} = Q^{Before}+Q^{After}$) and the y-axis reads the position of the limit order in the queue, i.e. $Q^{Before}$. The color and numbers gives the control associated to a pair $(Q^{same}, Q^{Before})$: 1 (blue) means ``stay in the book'', while 0 (red) means ``cross the spread'' to obtain a transaction. The panel (at the left) gives the reference optimal controls obtained with a finite difference scheme, the middle panel the optimal corresponds to the controls obtained for a \nameref{Alg:v_1_s} algorithm where the step-size $(\gamma_k)_{k\geq 0}$ is derived from the upper level policy, and the right panel the optimal control obtained with our optimal policy (i.e. upper level and inner level combined). It shows that after few iterations our optimal policy already found the optimal controls. Figure \ref{Fig:h_2_3_intro} compares the log of the $L^2$ error, averaged over 100 trajectories, between  the different algorithms. We see clearly that our methodology improves basic stochastic approximation algorithms. Again, the other algorithms behave as expected: SAGA is better than a constant learning rate that is better than the standard RL (at the beginning, since we know that asymptotically RL will drive the error to zeros whereas a constant learning rate does not).
\begin{figure}[ht!]
\centering
    \hfill a) Theoretical optimal control \hfill b) step\_cste optimal control \hfill	 c) PASS optimal control  \hfill ~\\
        \includegraphics[width=0.25\linewidth]{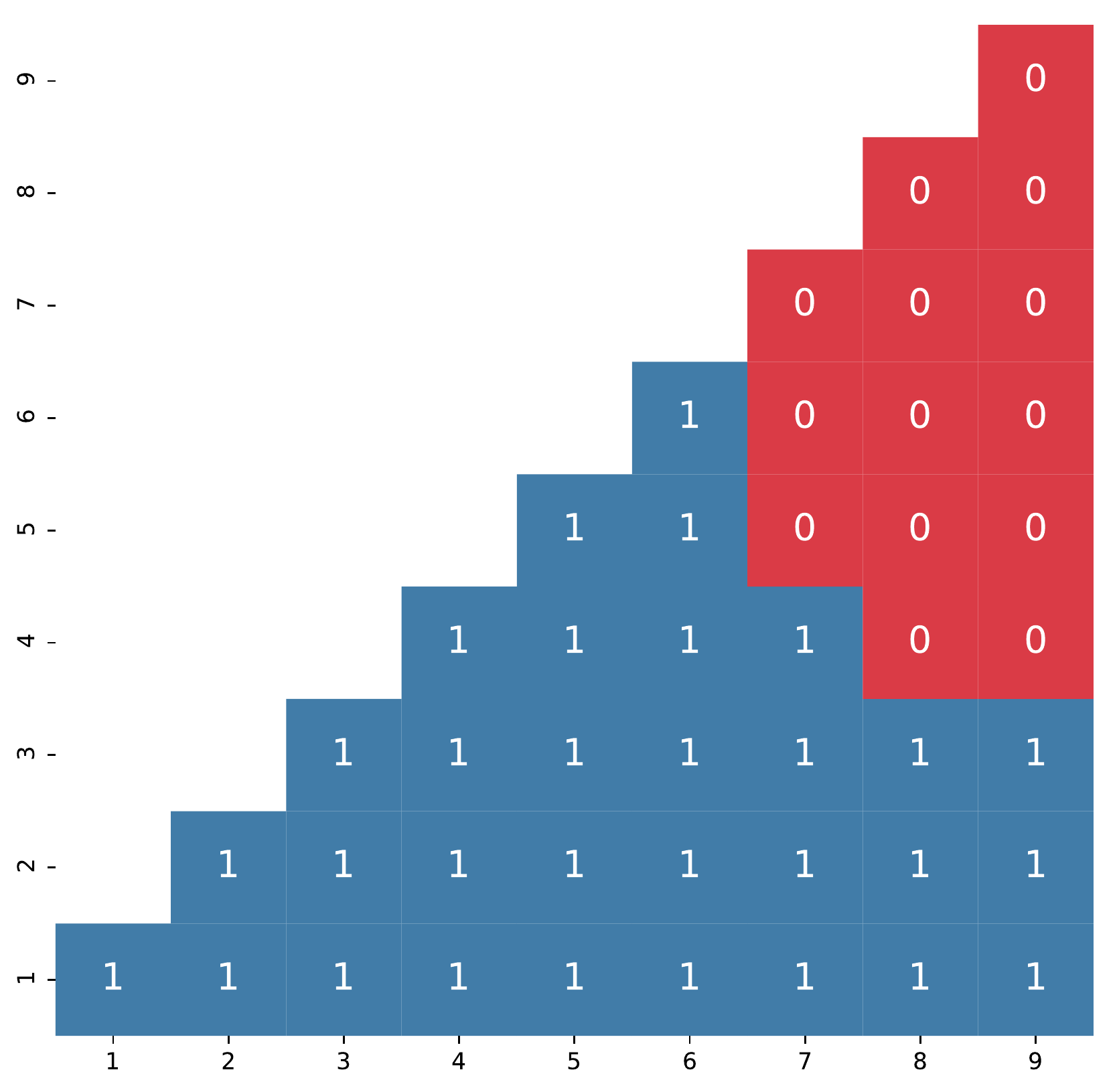}
	    \includegraphics[width=0.25\linewidth]{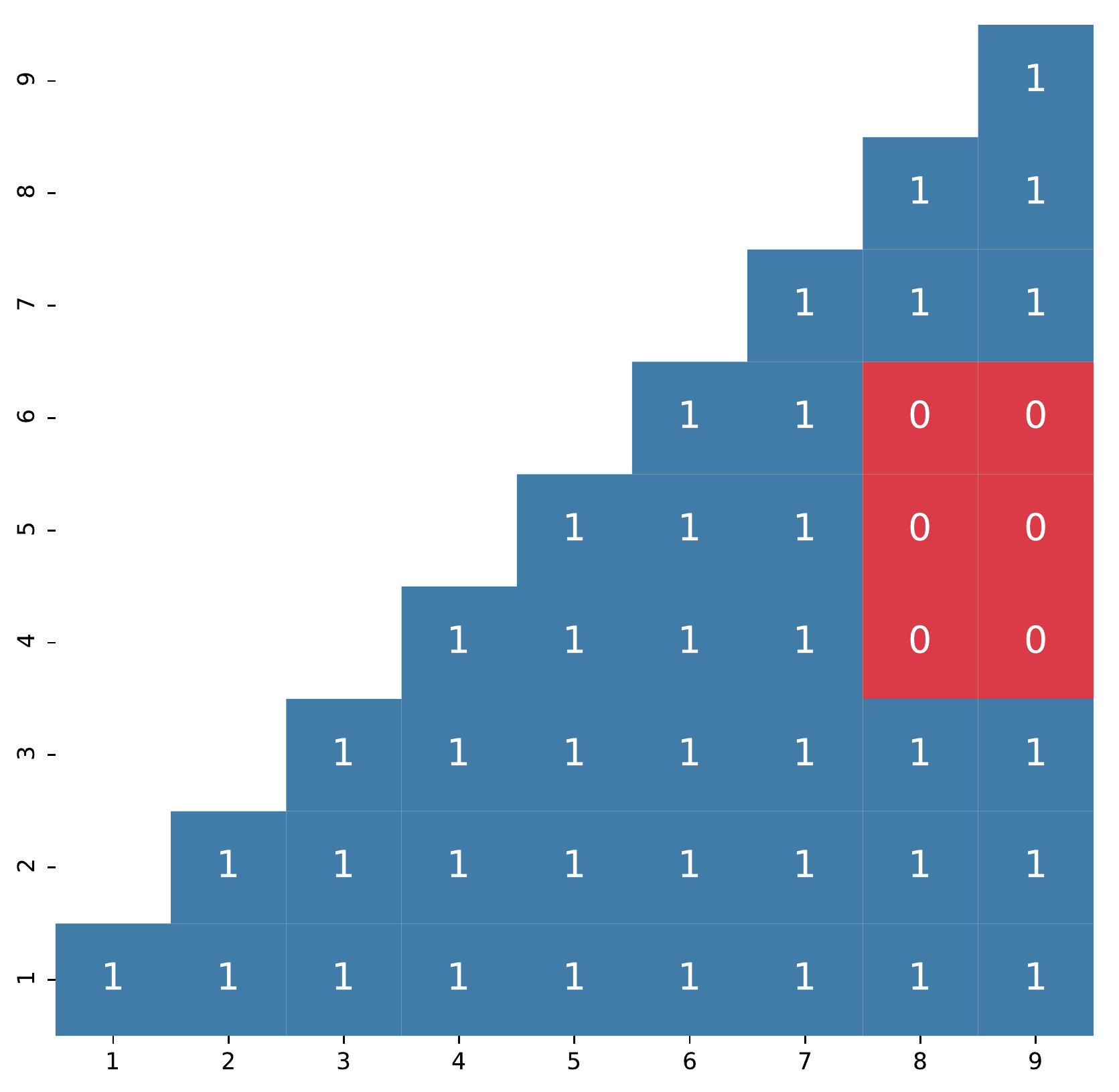}
	    \includegraphics[width=0.25\linewidth]{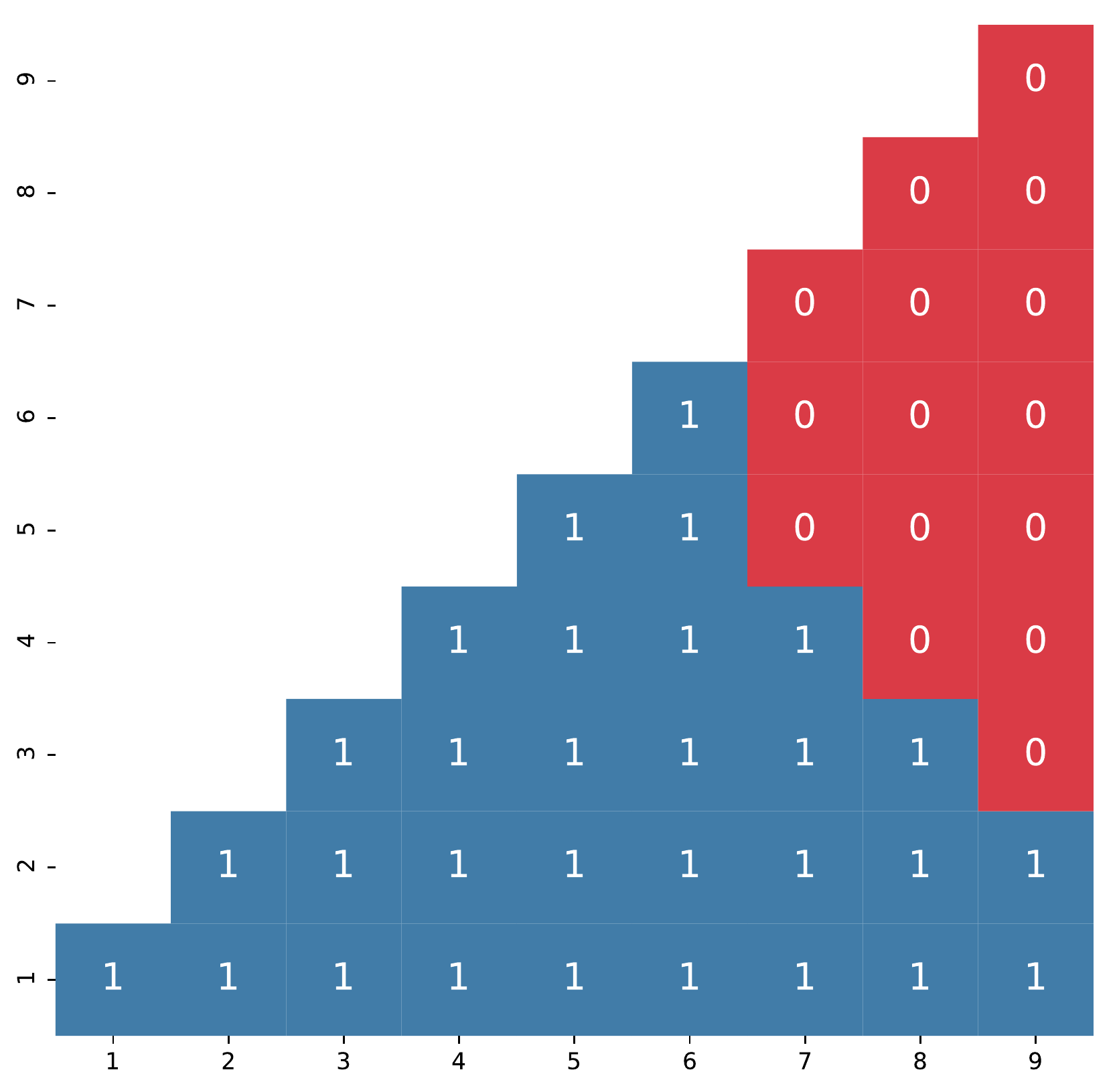}\\
	    \caption{Comparison optimal control after 300 iteration for different methods: left is the optimal control, middle is \nameref{Alg:v_1_s} with a step size derived from the upper level and right is our optimal policy for the step size (i.e. upper level and inner level combined).}
	     \label{Fig:h_2_2_intro}
\end{figure}
\begin{figure}[H]
\centering
    \hfill $L^2$-error against the number of iterations \hfill ~\\
        \includegraphics[width=0.65\linewidth]{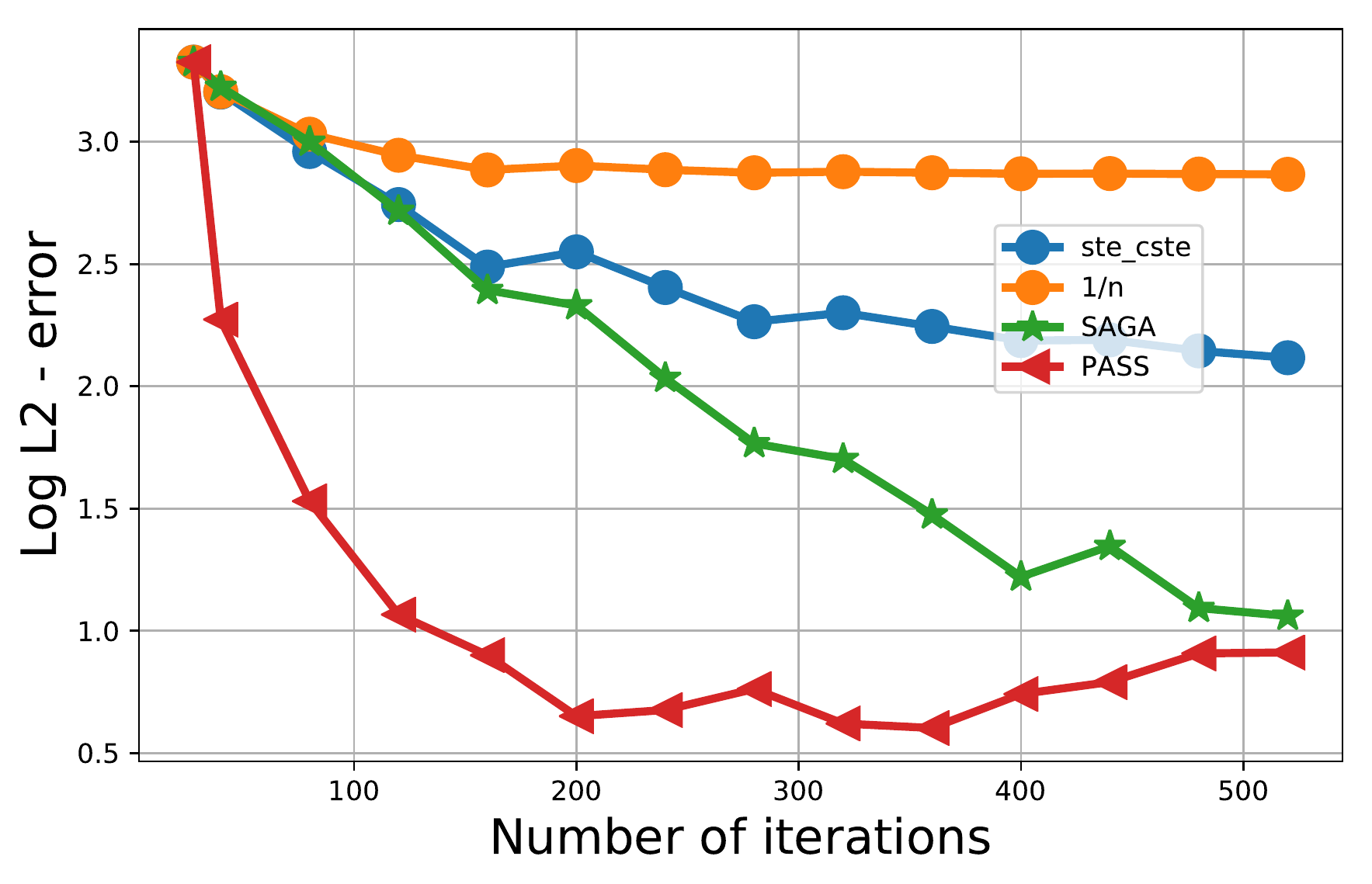}\\
	    \caption{The log $L^2$-error against the number of iterations averaged over 1000 simulated paths.}
	     \label{Fig:h_2_3_intro}
\end{figure}

\subsection{Optimal execution}
\label{sec:optim:cj}

\paragraph{Formalisation of the problem.} This is not the first work where RL is used to solve optimal trading problems. For example, authors in \cite{nevmyvaka2006reinforcement} apply RL techniques to solve optimal execution issues and in \cite{manziukAlgorithmicTradingReinforcement2019} they use deep reinforcement learning to solve a high dimensional market making problem. However, we consider here a different application. An investor wants to buy a given quantity $q_0$ of a tradable instrument (see \cite{cartea15book} and \cite{lehalleMarketMicrostructurePractice2018} for details about this setting). The price $S_t$ of this instrument satisfies the following dynamic:
\begin{align*}
dS_t = \alpha dt + \sigma dB_t, \numberthis \label{Eq:opti:cj:s_dynam}
\end{align*}
where $\alpha \in \mathbb{R}$ is the drift and $\sigma$ is the price volatility. The state of the investor is described by two variables its inventory $Q_t$ and its wealth $X_t$ at time $t$. The evolution of these two variables reads 
\begin{align*}
\left\{
\begin{array}{ll}
dQ_t = \nu_t dt, & Q_0 = q_0,\\
dW_t = - \nu_t (S_t + \kappa \nu_t) dt, & W_0 = 0,
\end{array}
\right.
\numberthis \label{Eq:opti:cj:inv_wealth}
\end{align*}
with $\nu_t$ the trading speed of the agent and $\kappa>0$. The term $\kappa \nu_t$ corresponds to the temporary price impact. The investor wants to maximize the following quantity
\begin{align*}
W_T + Q_T (S_T - A Q_T) - \phi \int_{t}^T Q_s^2 \, ds,
\end{align*}
it represents its final wealth $X_T$ at time $T$, plus the value of liquidating its inventory minus a running quadratic cost. The value function $V$ is defined such that
\begin{align*}
V(t,w,q,s) = \sup_{\nu} \Esp\big[ W_T + Q_T (S_T - A Q_T) - \phi \int_{t}^T Q_s^2 \, ds |W_t = w,\,Q_t = q,\,S_t =s\big].
\end{align*}
We remark that $v(t,w,q,s) = V(t,w,q,s) - w -qs $ verifies
\begin{align*}
v(t,w,q,s) = \sup_{\nu} \Esp\big[ \underbrace{(W_T - W_t) + (Q_T S_T - Q_t S_t)}_{ = M_T^t} - A Q^2_T - \phi \int_{t}^T Q_s^2 \, ds |W_t = w,\,Q_t = q,\,S_t =s\big]. 
\end{align*}
Using \eqref{Eq:opti:cj:s_dynam}, and \eqref{Eq:opti:cj:inv_wealth}, we can see that the variable $M_T^t$ is independent of the initial values $W_t$, and $S_t$. This means that $v$ is a function of only two variables the time $t$ and the inventory $q$. The dynamic programming principle ensures that $v$ satisfies 
\begin{align*}
v(t,q) = \sup_{\nu} \Esp\big[ M_{t+\Delta}^t - \phi \int_{t}^{t+\Delta} Q_s^2 \, ds + v(t+\Delta,Q_{t+\Delta}) |Q_t = q\big].\numberthis \label{Eq:v_dpp}
\end{align*}
We fix a maximum inventory $\bar{q}$. Let $k = (k_T,k_q) \in (\mathbb{N}^*)^2$, $\Delta = T/k_T$, $D_T = \{t_i^{k_T};\,i\leq k_T\}$, and $D_q = \{q_i^{k_q};\,i\leq k_q\}$ with $t_i^{k_T} = i \Delta$ and $q_i^{k_q} = -\bar{q} + 2i \bar{q}/ k_q $. To estimate $v$ we use the numerical scheme $(v_n^k)_{n \geq 1,\,k\in(\mathbb{N}^*)^2}$ defined below:
$$
v_{n+1}^k(Z_n) = v_{n}^k(Z_n) + \gamma_n(Z_n) \big[ \sup_{\nu \in A(Z_n)} \{ M^\nu_{n+1} - \phi \Delta Q_n^2 + v^k_n(Z^\nu_{n+1})  - v^k_n(Z_n)\} \big],
$$
with $Z_n = (n \Delta, Q_{n\Delta})$ and $ A(Z_n) \in D_q$ is the set of admissible actions\footnote{We do not allow controls that lead to states where the inventory exceeds $\bar{q}$.}. When the final time $T$ is reached (i.e. $ n = k_T$), we pick a new initial inventory from the set $D_q$ and start again its liquidation. At a first sight, it is not clear that $v_n^k$ approximates $v$. However, we have the following result.

\begin{prop}
The sequence $(v^k_n)_{n \geq 1, k \geq 1}$ converges point-wise towards $v$ on $D_T \times D_q$ when $n \rightarrow \infty$ and $k \rightarrow \infty$.
\label{prop:conv_num_schem}
\end{prop}

The proof of Proposition \ref{prop:conv_num_schem} is given in Appendix \ref{app:proof_v}. see Appendix \ref{sec:implementations_} for a detailed implementation of the algorithm with the corresponding pseudo-code (as Implementation \ref{Alg:PASS_opti_cj}).

\paragraph{Numerical results.} Figure \ref{Fig:opti_exec_v} shows the value function $v$ for different values of the elapsed time $t$ and the remaining inventory $Q_t$. The panel (at the left) gives the reference value function. It is computed by following the same approach of \cite{MFGTradeCrow2016}. The middle panel the value function obtained obtained after 120 000 iterations for \nameref{Alg:v_1_s} algorithm where the step-size $(\gamma_k)_{k \geq 0}$ is derived from the upper level of our optimal policy, and the right panel the value function obtained with our optimal policy (i.e upper level and inner level combined). It shows that our optimal strategy leads to better performance results. We also plot, in Figure \ref{Fig:opti_exec_error}, a simulated path for the variations of the log $L^2$ error for different algorithms. Here again, we notice that our methodology improves the basic RL algorithm and that the ordering of other approaches is similar to the one of the ``drift estimation'' approximation (i.e. SAGA and the constant learning rate are very similar).

 \begin{figure}[H]
 \centering
     \hfill a) Theoretical value function \hfill b) step\_ cste value function \hfill	 c) PASS value function  \hfill ~\\
         \includegraphics[width=0.3\linewidth]{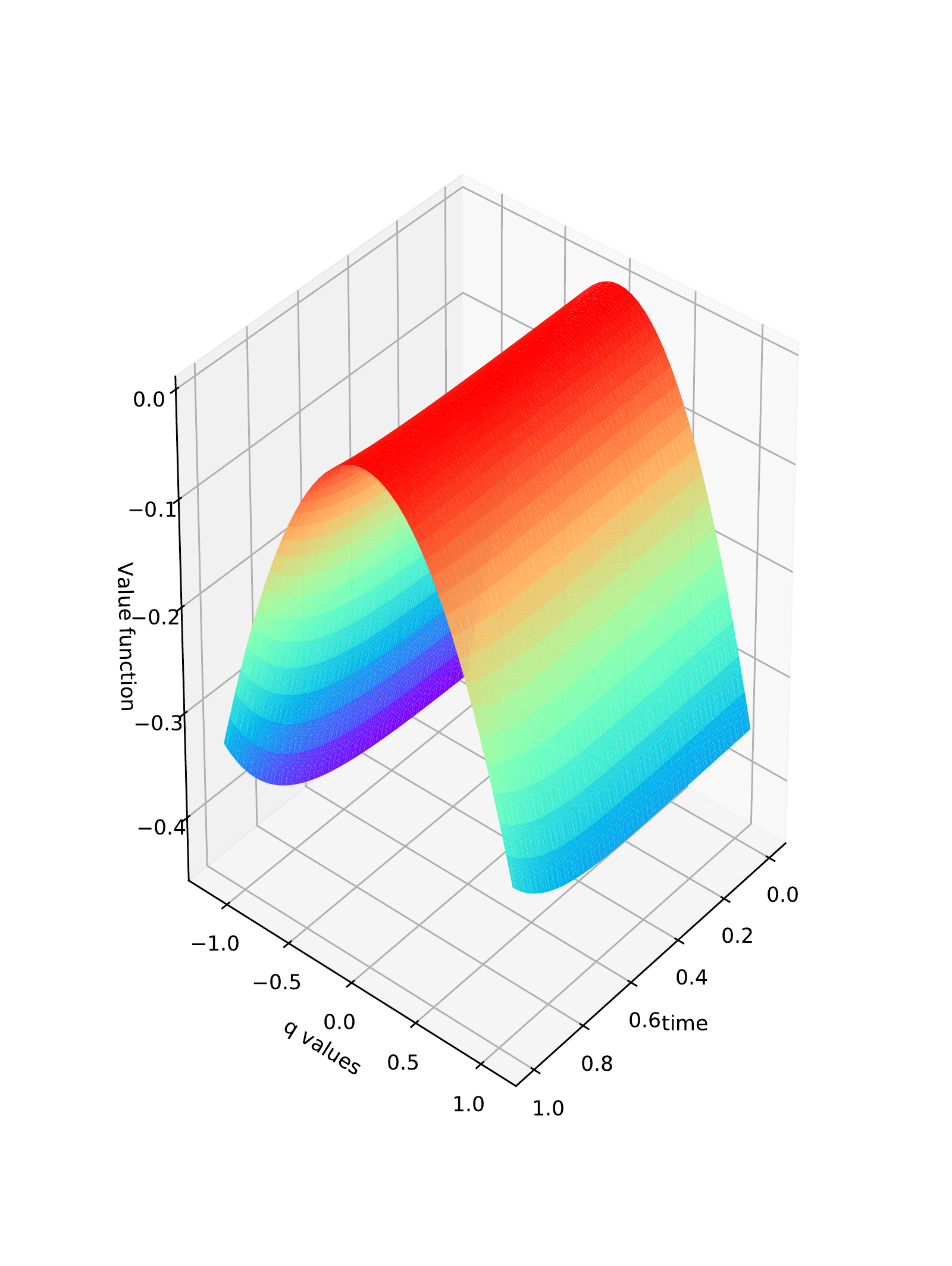}
 	    \includegraphics[width=0.3\linewidth]{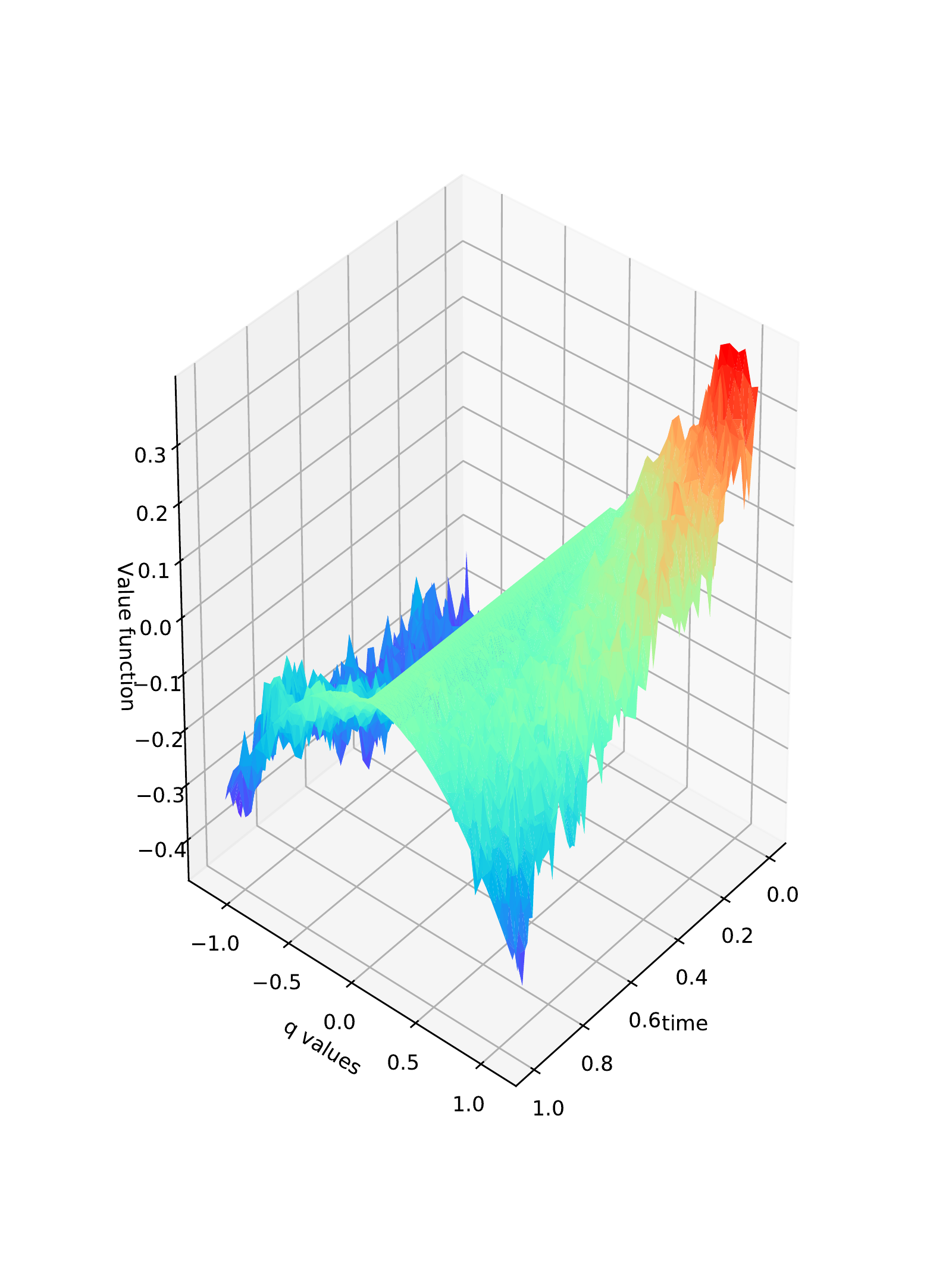}
 	    \includegraphics[width=0.3\linewidth]{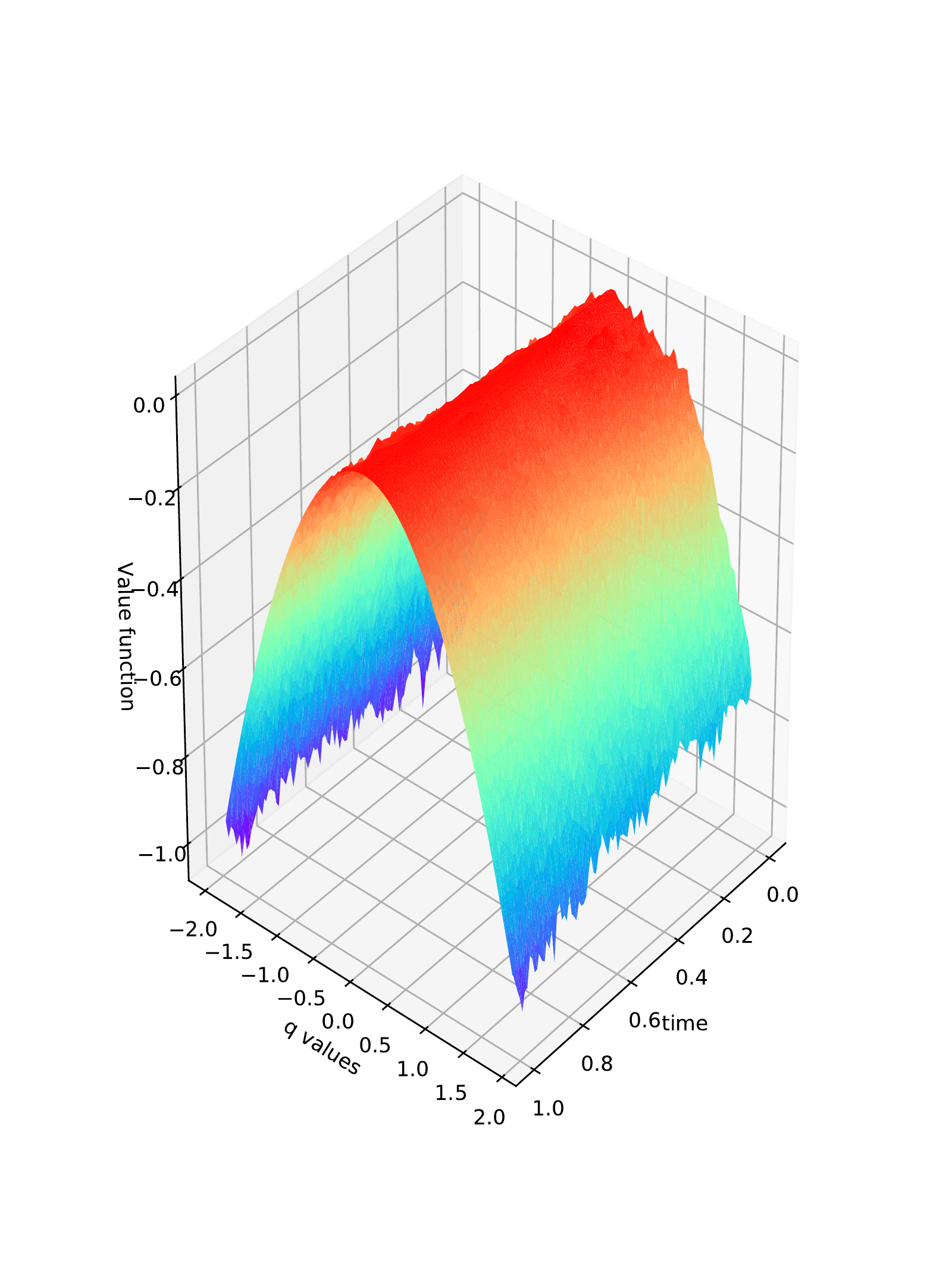}\\
 	    \caption{Comparison value function between methods.}
 	     \label{Fig:opti_exec_v}
 \end{figure}
 
\begin{figure}[H]
\centering
    \hfill $L^2$-error against the number of iterations \hfill ~\\
        \includegraphics[width=0.65\linewidth]{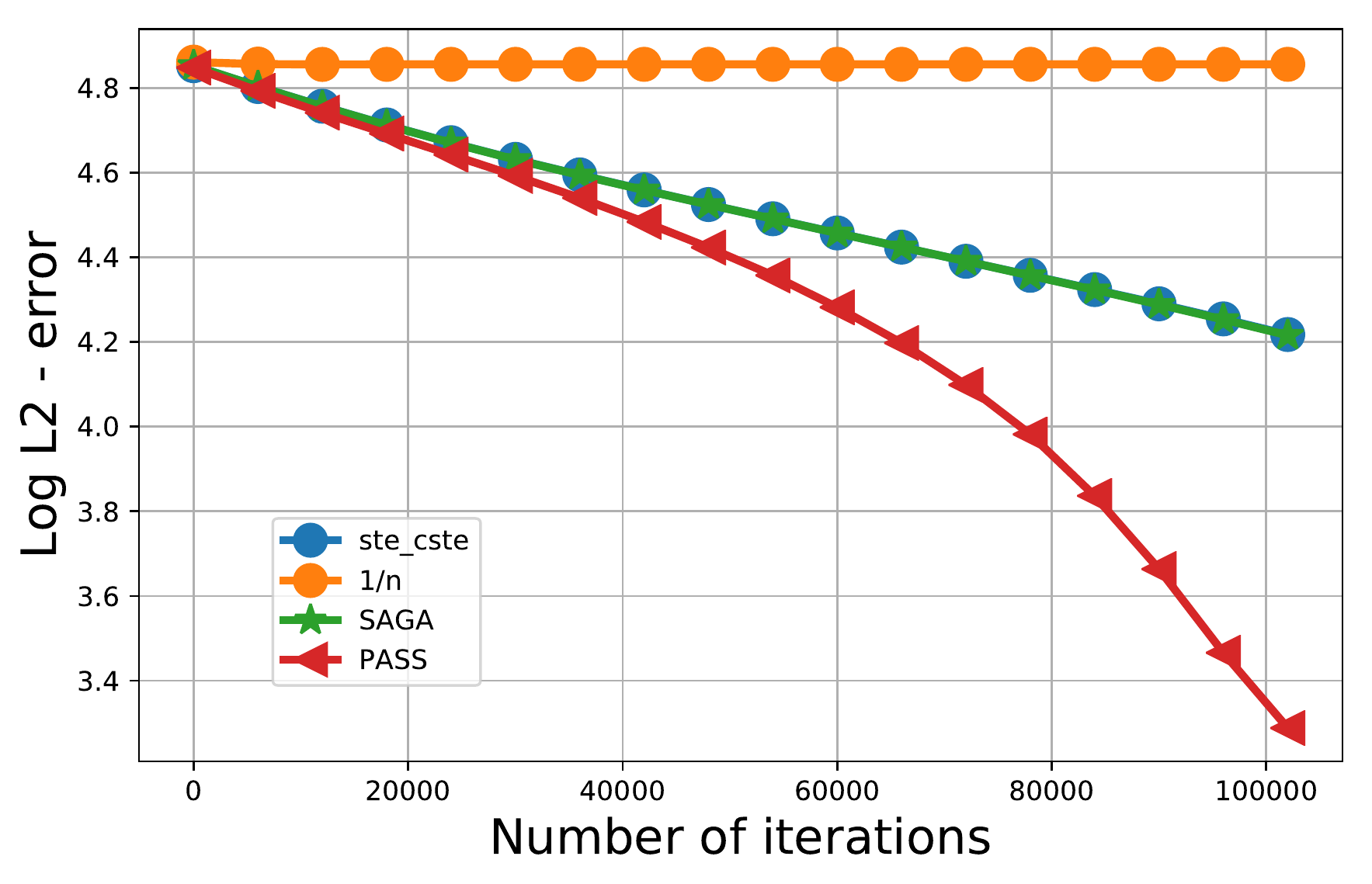}\\
	    \caption{The log $L^2$-error against the number of iterations averaged over 1000 simulated paths.}
	     \label{Fig:opti_exec_error}
\end{figure}

\clearpage
\appendix

\section{Actions of the agent}
\label{sec:actions_agent}
We present here policies that encourage exploration and others that maximize the agent's decisions.

\paragraph{Exploration policies:} Since it is important to visit the state space sufficiently enough, we propose to set the conditional distribution of the random variable $A_k$ such that
\begin{equation}
\mathbb{P}[A_k = a|\mathcal{F}_k] = \cfrac{e^{\bar{\epsilon}_k(Z^a_k)}}{\sum_{a'} e^{\bar{\epsilon}_k(Z^{a'}_k)}}, \qquad \forall a \in \mathcal{A},
\label{Eq:Policy_1}
\end{equation}
with $Z^a_k = (k,U_k,a)$, $\bar{\epsilon}_k(Z^a_k) = \beta(Z^a_k)\sum_{a'}  \epsilon_k(U^{Z^a_k}_{k+1},a')$ and 
$$
\epsilon_k(z) = \left\{
\begin{array}{ll}
|m(q_{r_1(z)},X_{r_1(z) +1}(z),z)| & \text{when the state } z \text{ already visited},\\
b & \text{otherwise,}
\end{array}
\right.
$$
where $b>0$ encourages the exploration, $q_k$ satisfies \eqref{Eq:RobinM_RL} and $r_1(z)$ is the last observation time of the state $z$.

\paragraph{Optimal policies.} To give more importance to the maximizing action, one may consider the following policy:
\begin{equation}
\mathbb{P}[A_k = a|\mathcal{F}_k] = \cfrac{e^{\beta_k(U_k) q_{k}(Z^a_k)}}{\sum_{a'} e^{\beta_k(U_k) q_{k}(Z^{a'}_k)}}, \qquad \forall a \in \mathcal{A}.
\label{Eq:Policy_2}
\end{equation}
Any mixture of these two procedures can an also be considered.

\section{Implementations}
\label{sec:implementations_}
We give here the pseudo code used for each one of the three numerical examples considered in Section \ref{sec:examples}.
\paragraph{Drift estimation.} We consider the following expression for the functions $h$ and $l$:
$$
h(\gamma,\gamma_{base}) = \min(\gamma + \gamma_{base},3\gamma_{base}), \quad l(\gamma,\gamma_{base}) = \max(\gamma - \gamma_{base},\gamma_{base}), \quad \forall (\gamma,\gamma_{base}) \in \mathbb{R}_+.
$$
We use Implementation \ref{Alg:PASS_drift_estim} for the numerical experiments.
\begin{algorithm}
\caption{PAst Sign Search (PASS) for (RL) drift estimation problem}
\begin{algorithmic}[1]
\State Algorithm parameters: step size $(\gamma^{o})_{n\geq 0} \in (0,1]$, number of episodes $n$
\Statex initial guess $q_0$, past error value $E_{past}$
\Statex Initialise $\hat{\gamma}_0 = \gamma^{o}_0$
\For{episode in $1:n$}
	\For{ $t \in \{0,\ldots,n_{\max}-1 \}$}
			\State Observe $\Delta X_{next} = S_{t+1} - S_t$
			\If {the first visit time to $t$} 
    				\State $q_{0}(t) \leftarrow  q_{0}(t) - \hat{\gamma}_0(t) m(q_0,\Delta X_{next},t)$
			\ElsIf {$m(q_0,t,\Delta X_{next}) \times E_{past}(t) \geq 0$}
				\State $\hat{\gamma}_0(t) \leftarrow h\big(\hat{\gamma}_0(t),\gamma^o(t)\big)$
    				\State $q_{0}(t) \leftarrow  q_{0}(t) - \hat{\gamma}_0(t) m(q_0,\Delta X_{next},t)$
			\ElsIf {$m(q_0,t,\Delta X_{next}) \times E_{past}(t) <0$}
				\State $\hat{\gamma}_{0}(t) \leftarrow l\big(\hat{\gamma}_0(t),\gamma^o(t)\big)$
    				\State $q_{0}(t) \leftarrow  q_{0}(t) - \hat{\gamma}_0(t) m(q_0,\Delta X_{next},t)$
			\EndIf
			\State $E_{past}(t) \leftarrow m(q_0,t,\Delta X_{next})$
	\EndFor	
	\State Save the norm $\|E\|$ of the vector $E_{past}(t)$.
	\If {the average value of $\|E\|$ over the last $w = 5$ episodes is not reduced by $p = 1\%$}
		\State $\gamma^o(t) \leftarrow max\big(\gamma^o(t)/2,0.01\big)$	 (this is done each $w$ episodes)
	\EndIf
\EndFor
\end{algorithmic}
\label{Alg:PASS_drift_estim}
\end{algorithm}

\paragraph{Optimal placement of limit orders.} We consider the following expression for the functions $h$ and $l$:
\begin{equation}
h(\gamma,\gamma_{base}) = \max(\min(\gamma + 2/3\gamma_{base},3\gamma_{base}),\gamma_{base}), \quad l(\gamma,\gamma_{base}) = \max(\gamma - 2/3\gamma_{base},\gamma_{base}), \quad \forall (\gamma,\gamma_{base}) \in \mathbb{R}_+. \label{Eq:h_l_func}
\end{equation}
We use Algorithm \ref{Alg:PASS_opti_placement} for the numerical experiments. Note that we do not need to send a market order to know our expected future gain.
\begin{algorithm}
\caption{PAst Sign Search (PASS) for (RL) optimal placement problem}
\begin{algorithmic}[1]
\State Algorithm parameters: step size $(\gamma^{o})_{n\geq 0} \in (0,1]$, number of episodes $n$
\Statex initial guess $q_0$, past error value $E_{past}$
\Statex Initialise $\hat{\gamma}_0 = \gamma^{o}_0$
\For{episode in $n$}
	\State Select initial state $X_0$
	\For{each step within episode}
		\State Take the action stay in the order book
		\State Observe the new order book state $X_{next}$
		\For{a $\in \{0,1\}$}
			\If {the first visit time to $X_{next}$} 
    				\State $q_{0}(X_0,a) \leftarrow  q_{0}(X_0,a) - \hat{\gamma}_0(X_0,a) m^a(q_0,X_{next},X_0)$
			\ElsIf {$m^a(q_0,X_0,X_{next}) \times E_{past}(X_0,a) \geq 0$}
				\State $\hat{\gamma}_{0}(X_0,a) \leftarrow h\big(\hat{\gamma}_{0}(X_0,a),\gamma^o(X_0,a)\big)$
    				\State $q_{0}(X_0,a) \leftarrow  q_{0}(X_0,a) - \hat{\gamma}_0(X_0,a) m^a(q_0,X_{next},X_0)$
			\ElsIf {$m^a(q_0,X_0,X_{next}) \times E_{past}(X_0,a) <0$}
				\State $\hat{\gamma}_{0}(X_0,a) \leftarrow l\big(\hat{\gamma}_{0}(X_0,a),\gamma^o(X_0,a)\big)$
    				\State $q_{0}(X_0,a) \leftarrow  q_{0}(X_0,a) - \hat{\gamma}_0(X_0,a) m^a(q_0,X_{next},X_0)$
			\EndIf
			\State $E_{past}(X_0,a) \leftarrow m^a(q_0,X_{next},X_0)$
		\EndFor
		\State $X_0 \leftarrow  X_{next}$    		
	\EndFor
	\State Save the norm $\|E\|$ of the vector $E_{past}(t)$.
	\If {the average value of $\|E\|$ over the last $w = 40$ episodes is not reduced by $p = 5\%$}
		\State $\gamma^o(t) \leftarrow max\big(\gamma^o(t)/2,0.01\big)$	 (this is done each $w$ episodes)
	\EndIf
\EndFor
\end{algorithmic}
\label{Alg:PASS_opti_placement}
\end{algorithm}
\paragraph{Optimal execution of a large number of shares.} To solve this problem we use the same functions $h$ and $l$ considered in the previous problem, see \eqref{Eq:h_l_func}. Then, we apply Algorithm \ref{Alg:PASS_opti_cj}. In this problem, it is crucial to select actions according to the policy \eqref{Eq:Policy_1} in order to encourage exploration. The coefficient $\bar{\beta}$ used by the agent to select its actions is taken constant equal to $\bar{\beta}  = 5$. We consider the same policy for all the tested algorithms.\\
\begin{algorithm}
\caption{PAst Sign Search (PASS) for (RL) optimal execution problem}
\begin{algorithmic}[1]
\State Algorithm parameters: step size $(\gamma^{o})_{n\geq 0} \in (0,1]$, number of episodes $n$
\Statex initial guess $q_0$, past error value $E_{past}$
\Statex Initialise $\hat{\gamma}_0 = \gamma^{o}_0$
\For{episode in $n$}
	\State Select the initial inventory $Q_0$
	\For{$t \in \{0,\ldots,n_{T}-1 \}$}
		\State Observe the new price state $S_{next}$ and set $X_0 = (t,Q_0) $
		\State Observe the new price state $S_{next}$
		\If {the first visit time to $X_0$} 
    			\State $q_{0}(X_0) \leftarrow  q_{0}(X_0) - \hat{\gamma}_0(X_0) m(q_0,S_{next},X_0)$
		\ElsIf {$m(q_0,S_{next},X_0) \times E_{past}(X_0) \geq 0$}
			\State $\hat{\gamma}_{0}(X_0) \leftarrow h\big(\hat{\gamma}_{0}(X_0),\gamma^o(X_0)\big)$
    			\State $q_{0}(X_0) \leftarrow  q_{0}(X_0) - \hat{\gamma}_0(X_0) m(q_0,S_{next},X_0)$
		\ElsIf {$m(q_0,S_{next},X_0) \times E_{past}(X_0) <0$}
				\State $\hat{\gamma}_{0}(X_0) \leftarrow l\big(\hat{\gamma}_{0}(X_0),\gamma^o(X_0)\big)$
    				\State $q_{0}(X_0) \leftarrow  q_{0}(X_0) - \hat{\gamma}_0(X_0) m(q_0,S_{next},X_0)$
			\EndIf
			\State $E_{past}(X_0) \leftarrow m(q_0,S_{next},X_0)$
		\State Select an action $A$ and observe $Q_{next}$
		\State $Q_0 \leftarrow  Q_{next}$    		
	\EndFor
	\State Save the norm $\|E\|$ of the vector $E_{past}(t)$.
	\If {the average value of $\|E\|$ over the last $w = 300$ episodes is not reduced by $p = 1\%$}
		\State $\gamma^o(t) \leftarrow max\big(\gamma^o(t) - 0.01,0.01\big)$ (this is done each $w$ episodes)
	\EndIf
\EndFor
\end{algorithmic}
\label{Alg:PASS_opti_cj}
\end{algorithm}
\section{Proof of Proposition \ref{prop:prop_std_unif_bd}}
\label{app:proof_std_unif_bd}
\begin{proof}[Proof of Proposition \ref{prop:prop_std_unif_bd}] Let $z \in \mathcal{Z}$. Standard uniform convergence results ensure that
\begin{align*}
\Esp[ \sup_{q} \big|M(q,z) - M_n(q,z)\big| \big|n(z)] \leq c \frac{1}{\sqrt{n(z)\wedge 1}}, \qquad a.s.
\end{align*}
with $c>0$ a positive constant. Since the Markov chain $(Z_n)_{n\geq 1}$ is irreducible and the set $\mathcal{Z}$ is finite, the sequence $(Z_n)_{n\geq 1}$ is positive recurrent and we have 
$$
\cfrac{n(z)}{n} = \cfrac{\sum_{k=1}^n \mathbf{1}_{Z_k = z}}{n} \underset{n \rightarrow \infty}{\rightarrow} \mathbb{P}_{\mu}[Z_n = z]>0 = p(z), \quad a.s,
$$
with $\mu$ the unique invariant distribution of $(Z_n)_{n\geq 1}$. Thus, we have
$$
u_n(z) = \Esp\left[ \sqrt{\cfrac{n}{n(z)\wedge 1}}\right] \underset{n \rightarrow \infty}{\rightarrow} \cfrac{1}{\sqrt{p(z)}} > 0.
$$
This shows that $u_n(z)$ is bounded by a constant called $u_{\infty}(z) $ and ensures that 
\begin{equation}
\Esp[ \sup_{q} \big|M(q,z) - M_n(q,z)\big|(z)] \leq  c_1(z) \frac{1}{\sqrt{n}},
\label{eq:proof__std_unif_bd_0}
\end{equation}
with $c_1(z) = cu_{\infty}(z)$. Since $\mathcal{Z}$ is finite, we close the proof by summing Inequality \eqref{eq:proof__std_unif_bd_0} over all the coordinates $z$.
\end{proof}

\section{Proof of Proposition \ref{prop:prop_fast_unif_bd}}
\label{app:proof_fast_unif_bd}

\begin{proof}[Proof of Proposition \ref{prop:prop_fast_unif_bd}] Let $z \in \mathcal{Z}$.  We follow the same approach used in the proof of Proposition \ref{prop:prop_std_unif_bd} to get
$$
v_n(z) = \Esp\left[ \big(\cfrac{\log(\bar{n}(z))/\log(n)}{\bar{n}(z)/n}\big)^{\beta}\right] \leq \Esp\left[ \big(\cfrac{1}{\bar{n}(z)/n}\big)^{\beta}\right] \underset{n \rightarrow \infty}{\rightarrow} \ \cfrac{1}{p(z)^{\beta}} > 0.
$$
This shows that $v_n(z)$ is bounded by a constant $v_{\infty}(z) $ and ensures that 
$$
\Esp[ \sup_{q} \big|M(q,z) - M_n(q,z)\big|(z)] \leq  c_2(z) \left(\frac{\log(n)}{n}\right)^{\beta},
$$
with $c_2(z) = c'v_{\infty}(z)$. Using \eqref{Eq:fast_bound_cond}, and the same manipulations used in the proof of Proposition \ref{prop:prop_std_unif_bd} and Inequality \eqref{Eq:est_error_bound}, we complete the proof. 
\end{proof}

\section{Proof of Proposition \ref{prop:prop_0}}
\label{app:proof_prop_0}

\begin{proof}[Proof of Proposition \ref{prop:prop_0}] Let $k \geq 0$, $A$ be the set $A =  \{ Z_k = z \} $, and $m(q_k) \in \mathbb{R}^{\mathcal{Z}}$ such that $m(q_k)(z') = m(q_k,X_{k+1}(z'),z')$ for any $z' \in \mathcal{Z}$. We split the proof into three cases. In each one of these steps, we prove \eqref{Eq:prop_0_error} for a given algorithm.
\paragraph{Case (i):} In this step, we prove \eqref{Eq:prop_0_error} for Algorithm \nameref{Alg:v_1_s}. Let us fix $z \in \mathcal{Z}$. For simplicity, we forget about the dependence of $m$, and $q$ on $z$ and write respectively $m(q_k)$, and $q_k$ instead of $m(q_k)(z)$, and $q_k(z)$. We have
\begin{align*}
\mathbf{1}_{A} \Esp_k[(q_{k+1} - q^*)^2]  &  = \mathbf{1}_{A} \Esp_k[(q_{k} - \gamma_k m(q_k) - q^*)^2] \\
											   & =  \mathbf{1}_{A}\left\{  \big( q_k- q^* \big)^2 \underbrace{- 2 \gamma_k\Esp_k[m(q_k)](q_k- q^*)}_{ = (i)} + \gamma^2_k \underbrace{\Esp_k[m(q_k)^2]}_{ = (ii)} \right\}.\\
\end{align*}
Using Assumption \ref{assump:assump__1_s} and $\Esp_k[m^*] = 0$, we get $(i) \leq  - 2 L\gamma_k|q_k- q^*|^2$. Since $\Esp_k[m^*] = 0$, Assumption \ref{assump:assump__2_s} gives
\begin{align*}
(ii) &  = \Esp_k\big[\big(m(q_k) - \Esp_k[m(q_k)]\big)^2\big]  + \big(\Esp_k[m(q_k) - m^*]\big)^2\\
					 & \leq \Esp_k\big[\big(m(q_k) - \Esp_k[m(q_k)]\big)^2\big] + B( 1 + (q_k- q^*)^2).%
\end{align*}

We use now two independent copies of $X_k$ respectively denoted by $X^1$, and $X^2$. \footnote{Note that the dependence of $X^1$, and $X^2$ on $k$ is omitted since there is no possible confusion.} We also write $m(q_k)_X = m(q_k,X(z'),z')$ to emphasize the dependence of $m(q_k)$ on $X$. Using Jensen's inequality, and Assumption \ref{assump:assump__2_s}, we get 
\begin{align*}
\Esp_k\big[\big(m(q_k)_{X^1} - \Esp_k[m(q_k)_{X^2}]\big)^2\big] & = \Esp_k\big[\big(\Esp_k[m(q_k)_{X^1} - m(q_k)_{X^2}]\big)^2\big] \\
& \hspace{-0.7cm}\underbrace{\leq}_{\text{Jensen's inequality}} \Esp_k\big[\big(m(q_k)_{X^1} - m(q_k)_{X^2}\big)^2\big] \\
& \leq 3 \big( \Esp_k\big[\big(m(q_k)_{X^1} - m^*(q_k)_{X^1}\big)^2\big] +  \Esp_k\big[\big(m^*(q_k)_{X^2} - m^*(q_k)_{X^1}\big)^2\big] \\ 	
& \quad + \Esp_k\big[\big(m^*(q_k)_{X^2} - m^*(q_k)_{X^2}\big)^2\big] \big) \underbrace{\leq}_{\text{Assumption }4} 3 B(1 + v_k)
\end{align*}

Thus, we deduce that
\begin{align*}
\mathbf{1}_{A} \Esp_k[(q_{k+1} - q^*)^2(z)] & \leq \mathbf{1}_{A} \left\{(1\underbrace{- 2 \gamma_k L + B \gamma^2_k}_{= -p(\gamma_k)}) \big( q_k- q^* \big)^2(z) + \underbrace{\gamma^2_k(z) B(4 + 3v_k)}_{= M_k}\right\}, 
\end{align*}
which shows \eqref{Eq:prop_0_error} for Algorithm \nameref{Alg:v_1_s}.
\paragraph{Case (ii):} Here we show \eqref{Eq:prop_0_error} for Algorithm \nameref{Alg:v_2_s}. Let $z \in \mathcal{Z}$, and $\bar{M}^k(z)= \big(\sum_{j=1}^M M^k[z,j]\big)/M$. We forget here about the dependence of $m$, $q$, and $M$ on $z$ and write respectively $m(q_k)$, $q_k$, and $M^k[j]$ instead of $m(q_k)(z)$, $q_k(z)$,  and $M^k[z,j]$. Using $\Esp_k[m^*] = 0$ and $\Esp_k[M^k[i]] = \bar{M}^k$, we have
\begin{align*}
\mathbf{1}_{A} \Esp_k\big[\big(q_{k+1} - q^* \big)^2 \big] & \qquad = \mathbf{1}_{A}\left\{  ( q_k- q^* )^2 + 2   (\Esp_k[q_{k+1}] - q_k)(q_k- q^*) + \Esp_k[(q_{k+1} - q_k)^2] \right\}\\
		& \qquad =  \mathbf{1}_{A}\left\{ ( q_k- q^*)^2 - 2 (\gamma_k\Esp_k[m(q_k) - m^*])(q_k- q^*) \right.\\
		& \qquad \qquad \left. + \gamma_k^2\Esp_k\big[\big(m(q_{k}) - M^k[i] + \bar{M}^k\big)^2\big]\right\}\\
		& \underbrace{\leq}_{\text{ \normalfont Assumption \ref{assump:assump__1_s} }}\mathbf{1}_{A} \left\{ \big(1 - 2 L\gamma_k\big)\big(q_k- q^*\big)^2 + \gamma^2_k \underbrace{\Esp_k\big[\big(m(q^{k}) - M^k[i] + \bar{M}^k\big)^2\big]}_{= (1)}\right\}. \numberthis \label{Eq:Proof_of_lemma_4_0}
\end{align*}
We first dominate the term (1). Since $\Esp_k[m^*](z) = 0$ and 
$$
\Esp_k\left[\big(M^k[i] - m^*\big)^2\right]  =  1/M \sum_{j} \Esp_k\left[\big(M^k[j] - m^* \big)^2\right],
$$
we have
\begin{align*}
(1) & =  \Esp_k \left| \big(m(q_k) - \Esp_k[m^*]\big) - \big(M^k[i] - m^*\big) + \big(1/M\sum_{j} (M^k[z,j] - m^*(z) ) \big)\right|^2 \\
						 & \leq 3 \bigg[\Esp_k\left[ m(q_k) - m^*\right]^2 + \Esp_k\left[\big(M^k[i] - m^*\big)^2 \right]+ \Esp_k\bigg[\big(1/M\sum_{j} (M^k[j] - m^*\big) \bigg]^2 \bigg] \\
						 & \underbrace{\leq}_{ \text{Jensen's inequality}} 3 \bigg[\Esp_k\left[ \big(m(q_k) - m^*\big)^2 \right] + \Esp_k\left[\big(M^k[i] - m^*\big)^2\right] + 1/M\sum_{j} \Esp_k\left[\big(M^k[j] - m^* \big)^2\right]\bigg] \\
						 &  = 3 \bigg[  \Esp_k\left[ \big(m(q_k) - \Esp_k[m(q_k)]\big)^2\right] + \left( \Esp_k[m(q_k) - m^*]\right)^2 + 2/M \sum_{j} \Esp_k\left[\big(M^k[j] - m^* \big)^2\right]\bigg]\\
						  &  \underbrace{\leq}_{\text{ Assumption \ref{assump:assump__2_s}}} 3 \bigg[ B\big( 4 + 3v_k + \left(q_k - q^*\right)^2\big) + 2/M \sum_{j} \Esp_k\left[\big(M^k[j] - m^* \big)^2\right] \bigg]\\
						 & = 3  B\big( 4 + 3v_k \big) + 3  B\left(q_k - q^*\right)^2 + 6/M \sum_{j}  \Esp_k\left[\big(M^k[j] - m^* \big)^2\right]\bigg]. \numberthis \label{Eq:Proof_of_lemma_4_1}
\end{align*}
By combining \eqref{Eq:Proof_of_lemma_4_0} and \eqref{Eq:Proof_of_lemma_4_1}, we get 
\begin{align*}
\mathbf{1}_{A} \Esp_k[|q_{k+1} - q^*|^2] & \leq  \mathbf{1}_{A} \left\{(1- 2 \gamma_k L + 3B \gamma^2_k) \big( q_k- q^*\big)^2  \right. \\
										 & \left. + 6/M \sum_{j} \Esp_k\left[\big(M^k[j] - m^* \big)^2\right] + 3 B \gamma^2_k (3 v_k + 4) \right\}.\numberthis \label{Eq:Proof_of_lemma_4_2}
\end{align*}
Moreover, we have 
\begin{align*}
\mathbf{1}_{A} 1/M \Esp_k\big[\sum_{j=1}^M \big(M^{k+1}[j]- m^*\big)^2\big] & =  \mathbf{1}_{A} \left\{\frac{1}{M} \Esp_k\big[(m(q_{k}) - m^*)^2\big] \right.\\
					& \qquad \left. + (1 - \frac{1}{M})	\frac{1}{M} \sum_{j=1}^M \Esp_k\big[\big(M^{k}[j]- m^*\big)^2\big] \right\}\\
 					& \leq \mathbf{1}_{A} \left\{\frac{B}{M}\big(1 + (q_{k} - q^*)^2\big) \right.\\
 					& \qquad \left. + (1 - \frac{1}{M}) \frac{1}{M}\sum_{j=1}^M \Esp_k\big[\big(M^{k}[j]- m^*\big)^2\big] \right\}.\numberthis \label{Eq:Proof_of_lemma_4_3}
\end{align*}
Thus using \eqref{Eq:Proof_of_lemma_4_2} and \eqref{Eq:Proof_of_lemma_4_3}, we conclude 
\begin{align*}
\mathbf{1}_A\Esp_k[e^{k+1}] & \leq \mathbf{1}_A \left\{ \underbrace{(1 - 2 \gamma_k L + 3B \gamma^2_k + \frac{B}{Mc})}_{ = \alpha_1} c\big( q_k- q^*\big)^2 \right. \\
							  & \quad \left. + \underbrace{(1 - \frac{1}{M} + 6 \gamma^2_k c)}_{\alpha_2} \frac{1}{M} \sum_{j} \Esp_k\big[\left(M^k[j] - m^* \right)^2 \big]\right\} + 3 B c \gamma^2_k (3v_k + 4)\\ 
                & \leq \mathbf{1}_A \alpha e^k + 3 c B \gamma^2_k (3v_k + 4),
\end{align*}
with $\alpha = \max(\alpha_1,\alpha_2)\in[0,1)$.
\paragraph{Case (iii):} In this final step, we show \eqref{Eq:prop_0_error} for Algorithm \nameref{Alg:v_4_s}. Here again, we forget about the dependence of the variables on $z$ as in the previous steps. We have
\begin{align*}
\mathbf{1}_{A} \Esp_k[(q_{k+1} - q^*)^2]  &  = \mathbf{1}_{A} \Esp_k[(q_{k} - \hat{\gamma}_k m(q_k) - q^*)^2] \\
											   & =  \mathbf{1}_{A}\left\{  \big( q_k- q^*\big)^2 \underbrace{- 2 \Esp_k[\hat{\gamma}_k m(q_k)](q_k- q^*)}_{ = (i)} +\underbrace{\Esp_k[\hat{\gamma}_k^2(m(q_k))^2]}_{ = (ii)} \right\}.
\end{align*}
For the term (i), using Assumption \ref{assump:assump__1_s} and $\Esp_k[m^*] = 0$, we have $(i) \leq - 2 c_k \bar{\gamma}_k (q_k- q^*)^2$ with $c_k = \frac{\Esp_k[\hat{\gamma}_k m(q_k)]}{\bar{\gamma}_k\Esp_k[m(q_k)]}$. Using Assumption \ref{assump:assump__2_s}, and $\Esp_k[m^*] = 0$, we get
\begin{align*}
(ii)  = c_k^2 \bar{\gamma}_k^2 \Esp_k\big[m(q_k)^2\big] & = c_k^2 \bar{\gamma}_k^2 \bigg(\Esp_k\big[\big(m(q_k) - \Esp_k[m(q_k)]\big)^2\big]  + \big(\Esp_k[m(q_k) - m^*]\big)^2\bigg)\\
					 & \leq  c_k^2 \bar{\gamma}_k^2 \bigg(B_k(4 + 3 v_k) +   B_k( 1 + (q_k- q^*)^2)\bigg).
\end{align*}
Thus, we deduce that
\begin{align*}
\mathbf{1}_{A} \Esp_k[(q_{k+1} - q^*)^2] & \leq \mathbf{1}_{A} \left\{(1\underbrace{- 2 c_k \bar{\gamma}_k L_k + B_k \big(c_k \bar{\gamma}_k\big)^2}_{=-p_k(c_k \bar{\gamma}_k)}) \big( q_k- q^* \big)^2 + \underbrace{c_k^2 \bar{\gamma}^2_k  B_k(4 + 3 v_k)}_{= M_k} \right\}. \numberthis \label{Eq:prop_step_4_1}
\end{align*}
We write $\underline{\gamma}_k$ for the quantity $\underline{\gamma}_k = c_k \bar{\gamma}_k \wedge \bar{\gamma}_k$. Since $ \hat{\gamma}_k \in [\gamma_k,r_1\gamma_k]$, we have $c_k\bar{\gamma}_k \in [\gamma_k,r_1\gamma_k]$. When $c_k \bar{\gamma}_k\in ]\gamma_k,\bar{\gamma}_k]$, we have $p_k(\underline{\gamma}_k) = p_k(c_k\bar{\gamma}_k) > p_k(\gamma_k) \geq 0$. When $c_k \bar{\gamma}_k \in ]\bar{\gamma}_k,r_1\gamma_k]$ (i.e $c_k \geq 1$), we use the following canonical decomposition of the function $p_k$:
$$
p_k(x) = B_k(x - \bar{\gamma}_k)^2 - \big(\cfrac{L_k}{B_k} - 1\big),
$$ 
and $\underline{\gamma}_k = \bar{\gamma}_k$ to get 
$$
|p_k(c_k \bar{\gamma}_k) - p_k(\underline{\gamma}_k)| = |p_k(c_k \bar{\gamma}_k) - p_k(\bar{\gamma}_k)| \leq B_k (c_k \bar{\gamma}_k - \bar{\gamma}_k)^2  \leq B_k (r_1\gamma_k - \gamma_k)^2 = \gamma_k^2 d^1,
$$
with $d^1 =  (r_1-1)^2 B_k$. Thus, using \eqref{Eq:prop_step_4_1}, we conclude
\begin{align*}
\mathbf{1}_{A} \Esp_k[(q^{k+1} - q^*)^2] & \leq \mathbf{1}_{A} \left\{(1-p_k(\underline{\gamma}_k) + \gamma_k^2 d^1\mathbf{1}_{c_k \geq 1}) \big( q^k- q^*\big)^2 + \underbrace{c_k^2 \bar{\gamma}^2_k(z) B_k(4 + 3v_k)}_{= M_k} \right\}. 
\end{align*}
This completes the proof.
\end{proof}
\section{Proof of Theorem \ref{theo:conv_error_numscheme_1}}
\label{app:proof_th_1}

For simplicity, the proof is split into two parts.

\subsection{Proof of Equation \eqref{eq:conv_error_numscheme_1_1}}

\begin{proof}[Proof of Equation \eqref{eq:conv_error_numscheme_1_1}]
Using Proposition \ref{prop:prop_0}, we get 

\begin{align*}
\Esp_n[e^{n+1}] & \leq e^n + \mathbf{1}_A \big( - \mu_n e^n + M_n \big) \\
        & \leq e^n - R_n + L_n,
\end{align*}
with $R_n  = \mathbf{1}_A \mu_n e^n$, and $L_n =  \mathbf{1}_A M_n$. Using the assumption $\sum_{n \geq 1} \gamma_n^2 < \infty$, and the expression of $M_n$, we obtain that $ \sum_{n \geq 1} L_n < \infty$. We can then apply the supermartingale convergence theorem, to deduce that $e^n$ converges towards a random variable with probability $1$, and 
$\sum_{n \geq 1} R_n < \infty$. Since $\sum_{n \geq 1} \gamma_n^2 < \infty$, we know that $\gamma_n$ converges towards $0$. Thus, for $n$ large enough we have 
$$
\mu_n \geq 2L \gamma_n - B \gamma^2_n \geq L \gamma_n.
$$
Replacing $R_n$, and $\mu_n$ by their expressions gives
$$
R_n = \mathbf{1}_A \mu_n e^n \geq \mathbf{1}_A L \gamma_n e^n,
$$
If there were $n_1$, and $\Delta>0$, such that 
$$
e^n \geq \Delta, \qquad \forall n \geq n_1,
$$ 
this would contradict the property $\sum_{n \geq 1} R_n < \infty$ since $\sum_{n \geq 1} \gamma_n = +\infty $.  Thus, we deduce that $e^n$ converges towards $0$.

\end{proof}

\subsection{Proof of Inequalities \eqref{eq:conv_error_numscheme_1_11} and  \eqref{eq:conv_error_numscheme_1_2}}

\subsubsection{Preparation for the proof of Inequality \eqref{eq:conv_error_numscheme_1_2}}
We introduce the following notations. Let $j \in \mathbb{N}^*$ and $(\mu_n,a_n,b_n)_{n\geq 1}$ be a sequence valued in $\mathbb{R}_+^3$. We write $(\mu,b)^j = (\mu^j_n,b^j_n)_{n\geq 1} $ for the delayed sequence $\mu^j_n = \mu_{j+n}$ and $b^j_n = b_{n+(j-1)}$ with $n\geq 1$. Additionally, we define recursively the sequence $(a^{\mu,b}_n)_{n\geq 1}$ as follows:
\begin{align*}
a^{\mu,b}_1 = 1,\qquad \text{ and } \qquad a^{\mu,b}_{n+1} =  \mu_{n+1}\sum_{l=1}^{n} a_{n+1-l}b_{l} a^{\mu,b}_{l}, \qquad \forall n \geq1. \numberthis \label{eq:seq_a_j_n}
\end{align*}
\begin{lem}By convention, an empty sum is equal to zero. Let $(v_n)_{n\geq1}$ be the sequence defined as follows:
\begin{align*}
v_{n} = \epsilon_n + \mu_n \big(\sum_{j=1}^{n-1}a_{n-j} b_j v_j \big) ,\qquad \forall n \geq 1,
\end{align*}
where $(\epsilon_n)_{n\geq 1}$ is a sequence. Then, we have 
\begin{align*}
v_n = \sum_{j=1}^n a^{(\mu,b)^j}_{n+1-j} \epsilon_j , \qquad \forall n\geq 1. \numberthis \label{Eq:lem_3_1}
\end{align*}
\label{lem:lem_5}
\end{lem}
\begin{proof}[Proof of Lemma \ref{lem:lem_5}]Let us prove the result by induction on $n \geq 1$. By definition, Equation \eqref{Eq:lem_3_1} is satisfied for $n=1$. By applying the induction hypothesis \eqref{Eq:lem_3_1} to all $ j \leq n$, we get 
\begin{align*}
v_{n+1} = 1 \times \epsilon_{n+1} + \mu_{n+1}  \big(\sum_{j=1}^{n} a_{n+1-j} b_j v_j \big) & = a^{(\mu,b)^{n+1}}_{1}\epsilon_{n+1} + \mu_{n+1} \big(\sum_{j=1}^{n} a_{n+1-j} b_j \sum_{l=1}^{j} a^{(\mu,b)^l}_{j+1-l} \epsilon_l \big)\\
	  & = a^{(\mu,b)^{n+1}}_{1}\epsilon_{n+1} + \big[\sum_{l=1}^{n} \mu_{n+1} \big(\sum_{j=l}^{n} a_{n+1-j} b_j a^{(\mu,b)^l}_{j+1-l} \big)\epsilon_l \big]\\
	  & = a^{(\mu,b)^{n+1}}_{1}\epsilon_{n+1} + \big[\sum_{l=1}^{n} \mu_{n+1} \big(\sum_{j=1}^{n-l+1} a_{n-l+2-j} b_{j+l-1} a^{\mu^l}_{j} \big)\epsilon_l \big]\\
	  & = a^{(\mu,b)^{n+1}}_{1}\epsilon_{n+1} + \sum_{l=1}^{n} a^{(\mu,b)^l}_{n+2- l} \epsilon_l = \sum_{j=1}^{n+1} a^{(\mu,b)^l}_{n+2 -j} \epsilon_j.
\end{align*}
\end{proof}
\begin{lem}Let $n \in \mathbb{N}$, $(a^{\mu,b}_n)_{n\geq 0}$ be the sequence defined in \eqref{eq:seq_a_j_n}, $(\mu_n)_{n\geq 0}$ be a positive non-decreasing sequence, and $r_n = 1 - \mu_n$.
\begin{itemize}
\item When $\sum_{n\geq 0} r_n = +\infty$, we have  
\begin{align*}
a^{\mu,b}_n \underset{n \rightarrow \infty}{\rightarrow} 0. \numberthis \label{Eq:lem_con_seq_1}
\end{align*}
\item There exists a non-negative constant $B$ such that 
\begin{align*}
a^{\mu,b}_n \leq B \sum_{k = 1}^{n-1} \bar{\mu}^n_{k} a^{*\infty}_{n-k} , \qquad \forall n \geq 2,\numberthis \label{Eq:lem_con_seq_0}
\end{align*}
with $\bar{\mu}^n_{k} = e^{- \sum_{j=n-k+1}^n r_j}$ and $a^{*\infty}_k = \lim_{n \rightarrow \infty} a^{*n}_k $. The sequence $(a^{* n}_k)_{k \geq 1}$ is defined recursively such that $a^{* 1}_k = a_k$ and $a^{* (n+1)}_k = \sum_{l = 1}^k a_{k+1-l} b_{n-k+l} a^{* n}_l$ for all $k \geq 1$.
\end{itemize}
\label{lem:lem_con_seq}
\end{lem}
\begin{proof}[Proof of Lemma \ref{lem:lem_con_seq}] \begin{itemize}[wide = 0pt, labelindent = 0pt]
\item \emph{Step 1.} We first prove by induction the following subsidiary inequality:
\begin{align*}
a^{\mu,b}_n  \leq \sum_{k = 1}^{n-1} \bar{\mu}^n_{k} \tilde{a}^{n}_{n-k}, \quad \forall n \geq 2, \numberthis \label{Eq:lem_con_seq_sub0}
\end{align*}
where $(\tilde{a}^n_k)_{k \geq 0}$ is a non-negative sequence that does not depend on $(\mu_n)_{n\geq 0} $. We also prove that Inequality \eqref{Eq:lem_con_seq_sub0} is optimal since it becomes an equality when the sequence $(\mu_n)_{n \geq 0}$ is constant. Inequality \eqref{Eq:lem_con_seq_sub0} holds clearly for $n = 2$ with $\tilde{a}^2 = a$. Using \eqref{eq:seq_a_j_n} and the induction assumption, we get
\begin{align*}
a^{\mu,b}_{n+1} = \mu_{n+1}\sum_{l =1}^n a_{n+1-l} b_l a^{\mu,b}_l & \leq \mu_{n+1} b_1 a_{n} + \mu_{n+1}\sum_{l =2}^n a_{n+1-l} b_l \sum_{k=1}^{l-1} \bar{\mu}^l_k \tilde{a}^l_{l-k}  \\
 & \leq \mu_{n+1} b_1  a_{n} + \mu_{n+1}\sum_{l =2}^n a_{n+1-l} b_l \sum_{k=1}^{l-1} \bar{\mu}^n_k \tilde{a}^l_{l-k} = (i) + (ii),\numberthis \label{Eq:step_1_proof_eq_0}
\end{align*}
where the inequality $\bar{\mu}^l_k \leq \bar{\mu}^n_k $ for any $l \leq n$ comes from the monotonicity assumption on $(\mu_n)_{n \geq 0}$. Note that $\bar{\mu}^l_k = \bar{\mu}^n_k$ when the sequence $(\mu_n)_{n \geq 0}$ is constant. In such case, we can replace all the previous inequalities by equalities. The inequality \eqref{Eq:step_1_proof_eq_0} gives 
\begin{align*}
(ii) = \sum_{k=1}^{n-1} \mu_{n+1}\bar{\mu}^n_k \sum_{l = k+1}^n a_{n+1-l} b_l \tilde{a}^l_{l-k} \leq \sum_{k=1}^{n-1} \bar{\mu}^{n+1}_{k+1} \sum_{l = 1}^{n-k} a_{n-k+1-l} b_{l+k} \tilde{a}^{k +l}_{l} & = \sum_{k=1}^{n-1} \bar{\mu}^{n+1}_{k+1} \tilde{a}^{n+1}_{n-k} \\
& = \sum_{k=2}^{n} \bar{\mu}^{n+1}_{k} \tilde{a}^{n+1}_{n+1-k},
\numberthis \label{Eq:step_1_proof_eq_0_Bis1}
\end{align*}
with
\begin{equation}
\tilde{a}^{n+1}_{k} = \sum_{l=1}^{k} a_{k+1-l} b_{n-k + l}\tilde{a}^{n-k+l}_{l}, \; \text{ when } k < n,\quad \tilde{a}^{n+1}_{n} = a_n, \label{Eq:step_1_proof_eq_1}
\end{equation}
and $\tilde{a}^{n+1}_k = 0$ otherwise. Combining \eqref{Eq:step_1_proof_eq_0}, and \eqref{Eq:step_1_proof_eq_0_Bis1} proves \eqref{Eq:lem_con_seq_sub0}. 

\item \emph{Step 2.} Here, we show that $\sum_{m \geq 1} \tilde{a}^{n}_m \leq 1$ for all $ n \geq 2$. To do so, let us consider the worst case where $\mu_n = 1$ for all $n\geq 0$. In such case, Inequality \eqref{Eq:lem_con_seq_sub0} is optimal which gives 
$
a_n = \sum_{k = 1}^{n-1} \tilde{a}^{n}_{n-k}.
$
It is easy to show using a direct induction, and \eqref{eq:seq_a_j_n} that $a_n \leq 1$. Thus, we deduce that $\sum_{m \geq 1} \tilde{a}^{n}_m \leq 1$.  
\item \emph{Step 3.} In this step, we prove \eqref{Eq:lem_con_seq_1}. Note that the condition $\sum_{n\geq 0} r_n = +\infty$ ensures that $ \bar{\mu}^n_k \underset{n \rightarrow \infty}{\rightarrow}  0$. Since $\sum_{m \geq 1} \tilde{a}^{n}_m \leq 1$, we can then apply the monotone convergence theorem to \eqref{Eq:lem_con_seq_sub0} and get \eqref{Eq:lem_con_seq_1}.
\item \emph{Step 4.} Let us prove by induction 
\begin{equation}
\tilde{a}^{\infty}_k = a^{*\infty}_k,\quad \forall k \geq 1, \label{Eq:step_4_proof_eq_1}
\end{equation}
with $\tilde{a}^{\infty}_k = \lim_{n \rightarrow \infty} \tilde{a}^n_k $, and $a^{*\infty}_k$ defined in \eqref{Eq:lem_con_seq_0}. Using \eqref{Eq:step_1_proof_eq_1}, and the definition of $a^{* n}_{1}$, we directly check that $\tilde{a}^{n}_1 = \tilde{a}^{*n}_1$ for any $ n \geq 2$. This proves $\tilde{a}^{\infty}_1 = a^{*1}_1$ and shows \eqref{Eq:step_4_proof_eq_1} for $k = 1$. Since $\sum_{l \leq n} a_{n+1-l} b_l \leq 1$ and $(\tilde{a}^n_l)_{n\geq 2,\,l \geq 1}$ is bounded by $1$, see Step 2, we can apply the dominated convergence theorem to \eqref{Eq:step_1_proof_eq_1} and get 
\begin{equation}
\tilde{a}^{\infty}_k = \sum_{l=1}^{k-1} a_{k+1-l} b_l \tilde{a}^{\infty}_{l} + a_1 b_{\infty}\tilde{a}^{\infty}_k. \label{Eq:step_3_proof_eq_1}
\end{equation}
When $a_1 b_{\infty} = 1$, we obtain $a_{j} b_{k'} = 0$ for any $k' \ne k$ or $j \ne 1$. Thus, a direct induction applied to \eqref{Eq:step_1_proof_eq_1} shows that $\tilde{a}^{n}_k = 0$ for any $k \geq 2$. By sending $n$ to $+\infty$, we find $\tilde{a}^{\infty}_k = 0$ for any $k \geq 2$ which guarantees \eqref{Eq:step_4_proof_eq_1}. We can then assume that $a_1 b_{\infty} < 1$. Under this assumption, Equation \eqref{Eq:step_3_proof_eq_1} reads  
\begin{equation}
\tilde{a}^{\infty}_k = \sum_{l=1}^{k-1} \underline{a}_{k+1-l} b_l\tilde{a}^{\infty}_{l}, \label{Eq:step_3_proof_eq_2}
\end{equation}
with $ \underline{a}_{l} = a_l/(1 - a_1 b_{\infty})$. The same lines of arguments show that $(a^{*\infty}_k)_{k \geq 1}$ also satisfies \eqref{Eq:step_3_proof_eq_2}. We can then apply the induction assumption to complete the proof \eqref{Eq:step_4_proof_eq_1}. 
\item \emph{Step 5.} Finally, we show \eqref{Eq:lem_con_seq_0}. Using \eqref{Eq:step_4_proof_eq_1}, we have $\lim_{n \rightarrow \infty} \bar{\mu}^n_{n-k} \tilde{a}^{n}_{k}/\bar{\mu}^n_{n-k} {a}^{* \infty}_{k} = 1$ for all $ k \geq 1$. This means that $(1-\epsilon) \bar{\mu}^n_{n-k} {a}^{* \infty}_{k} \leq \bar{\mu}^n_{n-k} \tilde{a}^{n}_{k} \leq (1+\epsilon) \bar{\mu}^n_{n-k} {a}^{* \infty}_{k}$ for any $\epsilon>0$ when $n$ becomes large enough. By summing the previous inequality over all the possible values of $k$ we get
\begin{equation}
\lim_{n \rightarrow \infty} \cfrac{\sum_{k = 1}^n \bar{\mu}^n_{n-k} \tilde{a}^{n}_{k}}{\sum_{k = 1}^n \bar{\mu}^n_{n-k} \tilde{a}^{* \infty}_{k}} = 1. \label{Eq:step_5_proof_eq_1}
\end{equation}
We combine \eqref{Eq:lem_con_seq_sub0}, and \eqref{Eq:step_5_proof_eq_1}, to show 
$
\lim_{n \rightarrow \infty} \cfrac{a^{\mu,b}_n}{\sum_{k = 1}^n \bar{\mu}^n_{n-k} \tilde{a}^{* \infty}_{k}} \leq \lim_{n \rightarrow \infty} \cfrac{\sum_{k = 1}^n \bar{\mu}^n_{n-k} \tilde{a}^{n}_{k}}{\sum_{k = 1}^n \bar{\mu}^n_{n-k} \tilde{a}^{* \infty}_{k}} = 1.
$
Thus, there exists $B\geq 0$ such that $a^{\mu,b}_n \leq B \sum_{k = 1}^n \bar{\mu}^n_{k} \tilde{a}^{* \infty}_{n-k}$ for all $n \geq 2$ which proves \eqref{Eq:lem_con_seq_0}.
\end{itemize}
\end{proof}

\subsubsection{Propagation of the error}
\manuallabel{subsec:Opti_policy}{E.2.2}
We introduce the following notations. Let $n \in \mathbb{N}^{*}$, and $z^1 \in \mathcal{Z}$. We write $\tau_{z_1} = \inf \{l>0,\, Z_{l} = z_1\}$, and $a_{k} = \mathbb{P}[\tau_{z_1 } \geq k, |Z_0 = z_1] $. We have the following result. \\
\begin{prop} Let $z_1 \in \mathcal{Z}$, and $n \in \mathbb{N}^{*}$. Under Assumptions \ref{assump:assump__0}, \ref{assump:assump__1_s}, \ref{assump:assump__2_s}, and \ref{assump:assump__3_s}, we have
\begin{align*}
\Esp_{0}[e^{n}(z_1)] \leq \epsilon_n + \sum_{j=1}^{n-1} a_{n - j}b_j\Esp_0[e^{j}(z_1)] ,
\end{align*}
with $\epsilon_n = e^1(z_1)a_n(z_1) + \sum_{j=1}^{n-1} a_{n-j}  \Esp[M_j]$ and $b_j = \cfrac{\Esp_0[\alpha_j \,  e^{j}(z_1)]}{\Esp_0[e^{j}(z_1)]}$. The variables $\alpha_j$, and $M_j$ are given by \eqref{Eq:val_eq_alpha}, and \eqref{Eq:val_eq_M}.
\label{prop:prop_1_2}
\end{prop}
\begin{proof}[Proof of Proposition \ref{prop:prop_1_2}] Using the last-exit decomposition, see Section 8.2.1 in \cite{meyn2012markov}, we have 
\begin{align*}
\Esp_{0}[e^n(z_1)] & = \Esp_{0}[e^n(z_1)\mathbf{1}_{\tau_{z_1} \geq n,A^n }] + \sum_{j=1}^{n-1} \Esp_{0}[e^n(z_1)\mathbf{1}_{\{\tau_{z_1} \geq n-j,\,Z_j =z_1\}}] \\
				   & = \Esp_{0}[e^1(z_1)\mathbf{1}_{\tau_{z_1} \geq n}] + \sum_{j=1}^{n-1} \Esp_{0}[e^{j+1}(z_1)\mathbf{1}_{\{\tau_{z_1} \geq n-j,\,Z_j =z_1\}}] \\
				   & \leq \Esp_{0}[e^1(z_1)\mathbf{1}_{\tau_{z_1} \geq n }] + \sum_{j=1}^{n-1} \mathbb{P}[\tau_{z_1 } \geq n-j|Z_j = z_1] \Esp_{0}[e^{j+1}(z_1) \mathbf{1}_{\{Z_j =z_1\}}] \\
				   & \underbrace{\leq}_{\text{Proposition \ref{prop:prop_0}}} \Esp_{0}[e^1(z_1)\mathbf{1}_{\tau_{z_1} \geq n,\,A^n}] + \sum_{j=1}^{n-1}  \mathbb{P}[\tau_{z_1 } \geq n-j|Z_j = z_1]  \Esp_{0}[(\alpha_j e^{j}(z_1) + M_j)] \\
				   					& \leq \epsilon_n + \sum_{j=1}^{n-1} a_{n - j} b_j \Esp_0[e^{j}(z_1)], 
\end{align*}
with $\epsilon_n$, $a_j$, and $b_j$ defined in Proposition \ref{prop:prop_1_2}. The variables $\alpha_n$, and $M_n$ are defined in Proposition \ref{prop:prop_0}. In the second equality, we use that $e^n(z_1)$ does not change as long as the state $z_1$ is not reached. This completes the proof.
\end{proof}
\subsubsection{Proof of Inequalities \eqref{eq:conv_error_numscheme_1_11} and  \eqref{eq:conv_error_numscheme_1_2}}
\begin{proof}[Proof of Theorem \ref{theo:conv_error_numscheme_1}] We split the proof in two steps. In Step (i), we show \eqref{eq:conv_error_numscheme_1_11} and then in Step (ii) we show \eqref{eq:conv_error_numscheme_1_2}.
\paragraph{Step (i):} In this part, we first prove \eqref{eq:conv_error_numscheme_1_11} when the space $\mathcal{Z}$ is finite. Then, we show how to extend \eqref{eq:conv_error_numscheme_1_1} to the general case. 
\paragraph{Sub-step (i-1):} Let us demonstrate \eqref{eq:conv_error_numscheme_1_1} when the space $\mathcal{Z}$ is finite. We define $v_n$ by $v_n = b_n \Esp[e^n(z_1)]$, and $r = \sum_{j \geq 1} a_{j} \leq \Esp[ \tau_{z_1} ] < \infty$. Using Proposition \ref{prop:prop_1_2}, the sequence $v_n$ verifies
\begin{align*}
v_n \leq \bar{\epsilon}_n + \mu_n \sum_{j=1}^{n-1} \bar{a}_{n-j} v_j(z_1), 
\end{align*}
with $\bar{\epsilon}_n(z_1) = b_n\epsilon_n$, $\mu_n = b_n$, and $\bar{a}_{j} = a_j / r$. Thus, using Lemma \ref{lem:lem_5}, we get 
\begin{align*}
v_n \leq \sum_{j=1}^n \bar{a}^{(\mu,b)^j}_{n+1-j} \bar{\epsilon}_j, \numberthis \label{Eq:reformulation_step1_1_th1}
\end{align*}
with $(\bar{a}^{(\mu,b)}_n)_{n \geq 1}$ defined in \eqref{eq:seq_a_j_n}.  We recall that $\epsilon_n = e^1 \bar{a}_n + \sum_{k=1}^{n-1}\bar{a}_{n-k}b_k\Esp[M_k]$.

Let us prove $\sum_{k \geq 1} \epsilon_k < \infty$. For this, we assimilate the sequence $\bar{\epsilon}$ to the measure $\mu = \sum_{k\geq 1} \bar{\epsilon}_k \delta_k$ with $\delta_k$ the Dirac measure at $k$ and recall that $\epsilon_n = e^1\bar{a}_n + \sum_{k=1}^{n-1}\bar{a}_{n-k}\Esp[M_k]$. Then, we introduce the variable $\tau_{z_1} = \inf \{l>0,\, Z_{l} = z_1\}$ already defined in Section \ref{subsec:Opti_policy}. By definition of $\bar{a}_{n-k}$, we have $\sum_{n \geq 1} \bar{a}_{n} = 1 < \infty$. Given that $\sum_{k} \Esp[\gamma^2_k] < \infty,$ and $\Esp[M_k]=O(\Esp[\gamma^2_k])$ for all the algorithms, we also get 
\begin{equation}
\sum_{k} \Esp[M_k] < \infty. \label{Eq:reformulation_step1_1_th1_002}
\end{equation}
We deduce from  \eqref{Eq:reformulation_step1_1_th1_002} that 
\begin{align*}
\Esp[\sum_{n\geq 1} \epsilon_k] & \leq e^1 \Esp[\sum_{n \geq 1} \bar{a}_{n} ] +  \sum_{k \geq 1} \Esp[M_k] \big( \Esp[\sum_{n \geq k} \bar{a}_{n-k} ]\big) \leq  e^1 + \sum_{k \geq 1} \Esp[M_k] < \infty.
\end{align*} 
This ensures that $\sum_{k} \epsilon_k < \infty$ and shows that the measure $\mu$ has a finite mass. 

Now, Lemma \ref{lem:lem_con_seq} gives $a^{(\mu,b)^j}_{n} \underset{n \rightarrow \infty}{\rightarrow} 0$, for any $j \geq 1$. Thus, the dominated convergence theorem ensures that $v_n\underset{n \rightarrow \infty}{\rightarrow} 0 $. Since $\alpha_n$ converges towards $1$, $b_n$ converges towards $1$ as well, which means that $\Esp[e^n(z_1)]  \underset{n \rightarrow \infty}{\rightarrow} 0 $. Since the space $\mathcal{Z}$ is finite, we deduce that $ \Esp[E^n] \underset{n \rightarrow \infty}{\rightarrow } 0$.
\paragraph{Sub-step(i-2):} In this second step, we show \eqref{eq:conv_error_numscheme_1_11} when the space $\mathcal{Z}$ is countable. Let $\epsilon>0$. Since $\Esp[E^1] < \infty$, there exists $k_0\in \mathbb{N}$ such that $$\sum_{k\geq k_0} \Esp[e^1(z_k)]\nu_k < \frac{\epsilon}{2}.$$
We write $A_{k_0}$ for the set $A_{k_0} = \{z_k,\; k\leq k_0\}$. Since $A_{k_0}$ is finite, we use \textbf{Sub-step(i-1)} to show the existence of $k_1\in \mathbb{N}$ such that $$\sum_{k\leq k_0} \Esp[e^{k_1}(z_k)]\nu_k < \frac{\epsilon}{2}, \quad \forall k\geq k_1, \forall z_k \in A_{k_0}.$$
We take now $k \geq k_1$. Using $(\Esp[e^{l}(z)])_{l \geq 1}$ is non-increasing for any $z \in \mathcal{Z}$, we get $E^{k} = \sum_{k'\geq k_0} \Esp[e^k(z_{k'})]\nu_{k'} + \sum_{k'< k_0} \Esp[e^k(z_{k'})]\nu_{k'} \leq \frac{\epsilon}{2} + \frac{\epsilon}{2} = \epsilon$.

\paragraph{Step (ii):} In this step, we show \eqref{eq:conv_error_numscheme_1_2}. By applying Lemma \ref{lem:lem_con_seq}, we obtain the existence of a constant $B$ such that
$$
a^{(\mu,b)}_n \leq B \sum_{k = 1}^{n-1} \bar{\mu}^n_{k} \bar{a}^{*\infty}_{n-k},
$$
with $\bar{\mu}^n_{k}$, and $\bar{a}^{*\infty}$ defined in Equation \eqref{Eq:lem_con_seq_0}. We can then use Equation \eqref{Eq:reformulation_step1_1_th1}, to get 
$$
v_n \leq B \sum_{j= 1}^{n} \sum_{l = 1}^j \bar{\mu}^n_{k} \bar{a}^{* \infty}_{l-j} \bar{\epsilon}_j.
$$

Since $\alpha_n $ converges towards $1$ almost surely, the variable $b_n$ is bounded from below and thus there exists a constant $B'$ such that 
\begin{equation}
\Esp[e^n(z_1) ] \leq B' \sum_{j= 1}^{n} \sum_{l = 1}^j \bar{\mu}^n_{k} \bar{a}^{* \infty}_{l-j} \bar{\epsilon}_j.\label{Eq:proof_th1_step2_01}
\end{equation} 
This completes the proof.

\end{proof}

\section{Proof of Proposition \ref{prop:opti_gamma}}
\label{sec:proof_speed_conv}
\begin{proof}[Proof of of Propostion \ref{prop:opti_gamma}] For simplicity, we prove the result for Algorithm \ref{Alg:v_1_s} however the same argument holds for Algorithm \ref{Alg:v_4_s}. Let us prove by induction on $n \in \mathbb{N}$ that 
\begin{align*}
e^{n}_{\gamma} \leq e^{n}_{\tilde{\gamma}}, \quad a.s, \quad \forall \tilde{\gamma} \in \Gamma. \numberthis \label{ineq:to_prove_induc}
\end{align*}

For $n = 0$, Inequality \eqref{ineq:to_prove_induc} is directly satisfied since the initial error $e^0$ does not depend on the choice of the learning rate. Let $\tilde{\gamma} \in \Gamma$. We assume now that $e^n_{\gamma} \leq e^n_{\tilde{\gamma}}$, $a.s.$ Using Equation \eqref{eq:dynamic_error_e} and the definition of $(\gamma_n)_{n \geq 0}$, we have  
\begin{align*}
e^{n+1}_{\gamma} = \mathbf{1}_{A} \big(g(e^{n}_{\gamma}) + S_n \big) + \mathbf{1}_{A^c}e^{n}_{\gamma} , \numberthis \label{ineq:to_prove_induc1}
\end{align*}

with $g(x) = x - \cfrac{L^2 x^2}{2B(x + (2 + v_n))}$. The previous expression of $g$ is obtained by minimising the function $y:\rightarrow (1 - 2L y + B y^2)e^{n}_{\gamma} + B(2 + v_n) \times y^2$. This means that 

\begin{align*}
g(x) \leq (1 - 2L y + B y^2)x + M_n y^2, \quad \forall y \in \mathbb{R}, \forall x \in \mathbb{R}_+.\numberthis \label{ineq:to_prove_induc2}
\end{align*}

A study of the function $g$ shows that it is non-decreasing on the interval $[0,x_2[$ with $x_2 = \arg\sup_{x \in \mathbb{R}_+} g(x)$ and $g(x) \leq x$ for any $x \in \mathbb{R}_+$. Thus, we have necessarily $\mathbf{1}_{A} e^n_{\tilde{\gamma}} \leq \mathbf{1}_{A} g(e^{n-1}_{\tilde{\gamma}}) \leq g(x_2) \leq x_2$ for any $n \geq 1$, and $\tilde{\gamma} \in \Gamma$. Note that when $x_2 = + \infty$, we do not need to assume that $\mathbf{1}_A S_n \leq 0$.  Using that $e^n_{\gamma} \leq e^n_{\tilde{\gamma}}$ a.s, the monotonicity of $g$ on $[0,x_2[$, and Equation \eqref{ineq:to_prove_induc2}, we get 

\begin{align*}
\mathbf{1}_{A}e^{n+1}_{\gamma} = \mathbf{1}_{A}g(e^{n}_{\gamma}) + \mathbf{1}_{A}S_n & \leq \mathbf{1}_{A}g(\mathbf{1}_{A}e^{n}_{\tilde{\gamma}}) + \mathbf{1}_{A}S_n \\
      & \leq (1 - 2L \tilde{\gamma}_n + B \tilde{\gamma}_n^2)\mathbf{1}_{A}e^{n}_{\tilde{\gamma}} + \mathbf{1}_{A}M_n \tilde{\gamma}_n^2 + \mathbf{1}_{A}S_n \\
      & = \mathbf{1}_{A}e^{n+1}_{\tilde{\gamma}}, \quad a.s \numberthis \label{ineq:to_prove_induc3}.
\end{align*}
We complete the proof by combining \eqref{ineq:to_prove_induc1}, \eqref{ineq:to_prove_induc3}, and the induction assumption.
\end{proof}

\section{Proof of Proposition \ref{prop:conv_num_schem}}
\label{app:proof_v}

\begin{proof}[Proof of Proposition \ref{prop:conv_num_schem}] We prove this result in three steps. First, we show that $v$ can be approximated by a numerical scheme $\bar{v}^{k}$. Then, we replace $\bar{v}^{k}$ by another scheme $v^k$ that also converges towards $v$. Finally, we show that $v^k_n$ tends to $v^k$ when $n \rightarrow \infty$. 
\paragraph{Step (i):} We start with our initial control problem where the agents may choose its trading speed at any time. It was studied by many authors, see for example \cite{MFGTradeCrow2016}, who show that the optimal trading speed verifies
\begin{align*}
\nu(t,q) = \cfrac{h_1(t) - q h_2(t)}{2\kappa}, \numberthis \label{Eq:eq_nu}
\end{align*}
with $h_1 : [0,T]\rightarrow \mathbb{R}$, and  $h_2 : [0,T]\rightarrow \mathbb{R}_+$ a positive function. This means that the optimal inventory follows:
$$
Q'(t) = \cfrac{h_1(t) - Q(t) h_2(t)}{2\kappa}, \quad \forall t \in [0,T],
$$
with $Q(0) = q_0$ given. Thus, $Q$ verifies 
\begin{align*}
Q(t) = q_0 e^{- \int_0^t \frac{h_2(s)}{2\kappa}\, ds} + e^{- \int_0^t \frac{h_2(s)}{2\kappa}\, ds} \int_{0}^t \cfrac{h_1(t)}{2\kappa}e^{- \int_0^s \frac{h_2(u)}{2\kappa}\, du}\, ds. \numberthis \label{Eq:eq_test_q}
\end{align*}
Using Equation \eqref{Eq:eq_test_q}, and basic inequalities, one can show that if we take a large enough $\bar{q} \in \mathbb{R}_+$ and place ourselves in $S_{q} = [-\bar{q},\bar{q}]$ then 
$$
Q(t) \in S_{q},
$$
when $Q(0) \in S_{q}$. Hence, we rewrite the dynamic programming principle as follows:
\begin{align*}
v(t,q) = \sup_{\nu \in \bar{A}(t,q) } \Esp\big[ M_{t+\Delta}^t - \phi \int_{t}^{t+\Delta} Q_s^2 \, ds + v(t+\Delta,Q_{t+\Delta}) |Q_t = q\big].\numberthis \label{Eq:v_dpp_v2}
\end{align*}
where $ \bar{A}(t,q) \subset S_q$ is the set of admissible actions\footnote{We only allow controls that lead to states where the inventory stays in $S_q$.}. We can focus on the set $ \bar{A}(t,q)$ instead of $\mathbb{R}$ since other controls are not optimal. Now, we approximate this problem in a classical way using the numerical scheme $\bar{v}^{k}$ defined such that 
\begin{align*}
\bar{v}^k(n_t,n_q) = \sup_{\nu \in D_a} \Esp\big[ M_{n_t+1}^{n_t} - \phi Q_{n_t\Delta_t}^2 \Delta  + \bar{v}^k(n_t+1,n^\nu_{q +1}) |Q_{n_t\Delta_t} =n_q\Delta_q\big], \quad \forall (n_t,n_q) \in D_T \times D_q, 
\end{align*}
with $M_{n_t+1}^{n_t} = M_{\Delta_t (n_t+1)}^{n_t \Delta}$, $n^\nu_{q +1}$ the index such that $ Q^\nu_{(n_t+1)\Delta_t} = n^\nu_{q +1} \Delta_q$ and $\bar{A} = A \cap \{i \Delta_q,\, i \in \mathbb{Z}\}$. The convergence of $(\bar{v}^k)_{k \geq \in (\mathbb{N}^*)^2}$ towards $v$ on the set $D_T \times D_q$ when $k \rightarrow \infty$ is standard.
\paragraph{Step (ii):} We denote by $v^k$ the numerical scheme 
\begin{align*}
v^k(n_t,n_q) =  \Esp\big[\sup_{\nu \in D_a}\big\{ M_{n_t+1}^{n_t} - \phi Q_{n_t\Delta_t}^2 \Delta  + v^k(n_t+1,n^\nu_{q +1})\big\} |Q_{n_t\Delta_t} =n_q\Delta_q\big], \quad \forall (n_t,n_q) \in D_T \times D_q, . 
\end{align*}
Let us show that $\bar{v}^k$ and $v^k$ have the same limit. For this, we use a backward recurrence. For the moment, we assume that $\big|\sup_{\nu} \Esp[M^{n_t}_{n_t+1}] - \Esp[\sup_{\nu}M^{n_t}_{n_t+1}]\big|  \leq K \Delta^2_t$ and we will prove it at the end of Step (ii). We want to show that $|\bar{v}^k(n_t,n_q) -  v^k(n_t,n_q)| \leq K (T-t) \Delta_t$ for all $(n_t,n_q) \in D_T \times D_q$. At the terminal time $\bar{v}^k$ and $v^k$ coincide. We move now to the induction part. We have
\begin{align*}
|\bar{v}^k(n_t,n_q) -   v^k(n_t,n_q)| & = \big| \sup_{\nu} \Esp[M^{n_t}_{n_t+1} - \phi q^2 \Delta_t + \bar{v}^k(n_t+1,n^{\nu}_{q_{t}+1})] \\
													      & \qquad - \Esp[\sup_{\nu} M^{n_t}_{n_t+1} - \phi q^2 \Delta_t + v^k(n_t+1,n^{\nu}_{q_t+1})]\big| \\
										  & \hspace{-1cm}\underbrace{\leq}_{\bar{v}\text{ and }v\text{ are not random}} \big|\sup_{\nu} \Esp[M^{n_t}_{n_t+1}] - \Esp[\sup_{\nu}M^{n_t}_{n_t+1}]\big| + \big|\sup_{\nu}\big\{\bar{v}^k(n_t+1,n^{\nu}_{q_{t}+1})- v^k(n_t+1,n^{\nu}_{q_{t}+1})\big\}\big|\\
													 & \leq \Esp\big[ \sup_{\nu}\big\{K \Delta_t^2  + K (T-t -\Delta_t) \Delta_t \big\}\big] \\
													 & =  K (T - t) \Delta_t.
\end{align*} 
In the third inequality, we use the induction assumption to complete the proof. Now, let us show that $\big|\sup_{\nu} \Esp[M^{n_t}_{n_t+1}] - \Esp[\sup_{\nu}M^{n_t}_{n_t+1}]\big|  \leq K \Delta^2_t$. For this, we write $t = n_t \Delta_t $, $\Delta Q = Q_{t+\Delta_t} - Q_t = \nu \Delta_t$, $\Delta S = S_{t+\Delta_t} - S_t$, and $\Delta \bar{S} = \int_{t}^{t + \Delta_t} (S_s - S_t) \, ds$. Thus, we have 
\begin{align*}
M^{n_t}_{n_t+1} &= (W_{t+\Delta} - W_t) +  (Q_{t+\Delta} S_{t+\Delta} - Q_t S_t) \\
			  & = -\nu \Delta \bar{S} - \Delta Q S_t - \kappa \nu^2 \Delta_t + Q_t \Delta	 S + \Delta Q S_t + \Delta Q \Delta S \\
			  & = -\nu \Delta \bar{S}  - \kappa \nu^2 \Delta_t + Q_t \Delta S + \nu\Delta_t \Delta S. \numberthis \label{Eq:M_dpp_opti_cj}
\end{align*}
The above equation shows $\sup_{\nu}  M^{n_t}_{n_t+1} = \frac{(\Delta_t \Delta S - \Delta \bar{S})^2}{4 \kappa \Delta_t} + Q_t \Delta S$. Using $\Esp[\Delta S] = \alpha \Delta_t$ and 
$$
\Esp[\Delta \bar{S}] = \int_{t}^{t+\Delta} \alpha (s-t) \, ds = \alpha\Delta_t^2/2,
$$
we get 
$$
\sup_{\nu} \Esp[M^{n_t}_{n_t+1}] =\frac{\alpha^2 \Delta_t^3}{16 \kappa} + \alpha Q_t \Delta_t , \quad \text{ and } \quad  \Esp[ \sup_{\nu} M^{n_t}_{n_t+1}] = \frac{\alpha^2 \Delta_t^3}{16 \kappa} + \frac{\sigma^2 \Delta^2_t}{12 \kappa}+ \alpha Q_t \Delta_t.
$$
Thus we deduce that $\big|\sup_{\nu} \Esp[M^{n_t}_{n_t+1}] - \Esp[\sup_{\nu}M^{n_t}_{n_t+1}]\big|  \leq K \Delta^2_t $ with  $ K = \frac{\sigma^2}{12 \kappa}$.
\paragraph{Step (iii):} Theorem \ref{theo:conv_error_numscheme_1}, proves that $v^k_n$ converges towards $v^k$. Thus by composition we have $v^k_n$ converges point-wise towards $v$ when $n \rightarrow \infty$ and $k \rightarrow \infty$ which completes the proof.
\end{proof}

\bibliographystyle{plain}
\bibliography{Learning_alg_sto}

\end{document}